\newif\ifshort
\shortfalse

\documentclass[a4paper]{article}

\usepackage[english]{babel}
\usepackage[utf8]{inputenc}
\usepackage[bookmarks]{hyperref}
\usepackage{amsmath}
\usepackage{amsthm}
\usepackage{amssymb}
\usepackage{enumerate}
\usepackage{tikz}
\ifshort\usepackage{wrapfig}\fi
\usepackage[labelfont=bf]{caption}
\usepackage{bbm}
\DeclareMathOperator*{\argmin}{arg\,min}
\DeclareMathOperator*{\argmax}{arg\,max}
\def\MDL{\mathrm{MDL}}          
\def\eps{\varepsilon}           
\def\epstr{\epsilon}            
\def\cinfty{\omega}             
\def\SetN{\mathbb{N}}           
\def\F{\mathcal{F}}             
\def\Foo{\mathcal{F}_\cinfty}   
\def\E{\mathbb{E}}              

\newenvironment{claims}
{\newcounter{ClaimCounter}}
{}
\newcommand{\newclaim}[1]{
\refstepcounter{ClaimCounter}
\emph{Claim \theClaimCounter: {#1}}
}

\usepackage{aliascnt}

\newcommand{\mynewtheorem}[2]{
	\newaliascnt{#1}{dummy}
	\newtheorem{#1}[#1]{#2}
	\aliascntresetthe{#1}
	\expandafter\def\csname #1autorefname\endcsname{#2}
}

\newcounter{dummyBackup}
\newcommand{\thmcountershouldbe}[1]{
	\setcounter{dummyBackup}{\thedummy}
	\setcounter{dummy}{#1}
}
\newcommand{\thmcounterbacktonormal}{
	\setcounter{dummy}{\thedummyBackup}
}

\theoremstyle{plain}
\mynewtheorem{theorem}{Theorem}
\mynewtheorem{lemma}{Lemma}
\mynewtheorem{corollary}{Corollary}
\theoremstyle{definition}
\mynewtheorem{definition}{Definition}
\mynewtheorem{remark}{Remark}

\AtBeginDocument{
	
}

\title{Indefinitely Oscillating Martingales}
\author{Jan Leike \and Marcus Hutter}
\date{\today}

\def\keywords{martingales, infinite oscillations, bounds, convergence rates,
  minimum description length, mind changes.}

\hypersetup{
	pdftitle={Indefinitely Oscillating Martingales},
	pdfauthor={Jan Leike and Marcus Hutter},
	pdfkeywords={\keywords},
	pdflang={en-US}
}

\begin{document}

\maketitle

\begin{abstract}
We construct a class of nonnegative martingale processes that oscillate indefinitely with high probability.
For these processes,
we state a uniform rate of the number of oscillations and
show that this rate is asymptotically close to the theoretical upper bound.
These bounds on probability and expectation of the number of upcrossings
are compared to classical bounds from the martingale literature.
We discuss two applications.
First, our results imply that the limit of the minimum description length operator may not exist.
Second, we give bounds on how often one can change one's belief in a given hypothesis when observing a stream of data.%
\ifshort
\else
\footnote{
This is the extended technical report.
The conference version can be found at \cite{LH:14martosc}.
}
\fi
\end{abstract}

\section{Introduction}\label{sec:introduction}

Martingale processes model fair gambles
where knowledge of the past or choice of betting strategy have no impact on future winnings.
But their application is not restricted to gambles and stock markets.
Here we exploit the connection between nonnegative martingales and
probabilistic data streams, i.e., probability measures on infinite strings.
For two probability measures $P$ and $Q$ on infinite strings,
the quotient $Q/P$ is a nonnegative $P$-martingale.
Conversely, every nonnegative $P$-martingale is a multiple of $Q/P$
$P$-almost everywhere
for some probability measure $Q$.

One of the famous results of martingale theory is \hyperref[thm:upcrossing-inequality]{Doob's Upcrossing Inequality}~\cite{Doob:53}.
The inequality states that in expectation,
every nonnegative martingale has only finitely many oscillations
(called \emph{upcrossings} in the martingale literature).
Moreover, the bound on the expected number of oscillations is inversely proportional to their magnitude.
Closely related is \hyperref[thm:Dubins-inequality]{Dubins' Inequality}~\cite{Dubins:62}
which asserts that the probability of having many oscillations decreases exponentially with their number.
These bounds are given with respect to oscillations of fixed magnitude.

In \autoref{sec:indefinitely-oscillating-martingales}
we construct a class of nonnegative martingale processes
that have infinitely many oscillations of (by Doob necessarily) decreasing magnitude.
These martingales satisfy uniform lower bounds
on the probability and the expectation of the number of upcrossings.
We prove corresponding upper bounds in \autoref{sec:upper-bounds}
showing that these lower bounds are asymptotically tight.
Moreover, the construction of the martingales is agnostic regarding the underlying probability measure,
assuming only mild restrictions on it.
We compare these results to the statements of \hyperref[thm:Dubins-inequality]{Dubins' Inequality} and
\hyperref[thm:upcrossing-inequality]{Doob's Upcrossing Inequality} and
demonstrate that our process makes those inequalities asymptotically tight.
If we drop the uniformity requirement,
asymptotics arbitrarily close to Doob and Dubins' bounds are achievable.
We discuss two direct applications of these bounds.

The Minimum Description Length (MDL) principle~\cite{Rissanen:78} and
the closely related Minimal Message Length (MML) principle~\cite{Wallace:68}
recommend to select among a class of models the one that
has the shortest code length for the data plus code length for the model.
There are many variations, so the following statements are generic:
for a variety of problem classes
MDL's predictions have been shown to converge asymptotically (predictive convergence).
For continuous independently identically distributed data
the MDL estimator usually converges to the true distribution~\cite{Gruenwald:07book,Wallace:05}
(inductive consistency).
For arbitrary (non-i.i.d.) countable classes,
the MDL estimator's predictions converge to those of the true distribution
for single-step predictions~\cite{Hutter:05mdl2px}
and $\infty$-step predictions~\cite{Hutter:09mdltvp}.
Inductive consistency implies predictive convergence,
but not the other way around.
In \autoref{sec:mdl-application} we show that indeed,
the MDL estimator for countable classes is \emph{inductively inconsistent}.
This can be a major obstacle for using MDL for prediction,
since the model used for prediction has to be changed over and over again,
incurring the corresponding computational cost.

Another application of martingales is in the theory of mind changes~\cite{Luo:05}.
How likely is it that your belief in some hypothesis changes by at least $\alpha > 0$
several times while observing some evidence?
Davis recently showed~\cite{Davis:13} using elementary mathematics
that this probability decreases exponentially.
In \autoref{sec:bounds-on-mind-changes} we rephrase this problem in our setting:
the stochastic process
\[
P( \,\text{hypothesis} \mid \text{evidence up to time $t$}\, )
\]
is a martingale bounded between $0$ and $1$.
The upper bound on the probability of many changes can thus be
derived from \hyperref[thm:Dubins-inequality]{Dubins' Inequality}.
This yields a simpler alternative proof for Davis' result.
However, because we consider nonnegative but unbounded martingales,
we get a weaker bound than Davis.

\section{Strings, Measures, and Martingales}
\label{sec:preliminaries}

We presuppose basic measure and probability theory~\cite[Chp.1]{Durrett:10}.
Let $\Sigma$ be a finite set, called \emph{alphabet}.
We assume $\Sigma$ contains at least two distinct elements.
For every $u \in \Sigma^*$,
the \emph{cylinder set}
\[
\Gamma_u := \{ uv \mid v \in \Sigma^\cinfty \}
\]
is the set of all infinite strings of which $u$ is a prefix.
Furthermore, fix the $\sigma$-algebras
\[
\F_t := \sigma \left(\{ \Gamma_u \mid u \in \Sigma^t \} \right)
\qquad\text{and}\qquad
\Foo := \sigma \Big( \bigcup_{t=1}^\infty \F_t \Big).
\]
$(\F_t)_{t \in \SetN}$ is a \emph{filtration}:
since $\Gamma_u = \bigcup_{a \in \Sigma} \Gamma_{ua}$,
it follows that $\F_t \subseteq \F_{t+1}$ for every $t \in \SetN$,
and all $\F_t \subseteq \Foo$ by the definition of $\Foo$.
An \emph{event} is a measurable set $E \in \Foo$.
The event $E^c := \Sigma^\cinfty \setminus E$ denotes the complement of $E$.
\ifshort\else
See also the list of notation
in Appendix \ref{app:notation}.
\fi

\begin{definition}[Stochastic Process]
\label{def:stochastic-process}
$(X_t)_{t \in \SetN}$ is called \emph{($\mathbb{R}$-valued) stochastic process} iff
each $X_t$ is an $\mathbb{R}$-valued random variable.
\end{definition}

\begin{definition}[Martingale]
\label{def:martingale}
Let $P$ be a probability measure over $(\Sigma^\cinfty, \Foo)$.
An $\mathbb{R}$-valued stochastic process $(X_t)_{t \in \SetN}$ is called a
\emph{$P$-supermartingale ($P$-sub\-martin\-gale)}
iff
\begin{enumerate}[(a)]
\item each $X_t$ is $\F_t$-measurable, and
\item $\E[X_t \mid \F_s] \leq X_s$ ($\E[X_t \mid \F_s] \geq X_s$) almost surely
	for all $s, t \in \SetN$ with $s < t$.
\end{enumerate}
A process that is both $P$-supermartingale and $P$-submartingale is called \emph{$P$-martingale}.
\end{definition}
We call a supermartingale (submartingale) process $(X_t)_{t \in \SetN}$ \emph{nonnegative} iff
$X_t \geq 0$ for all $t \in \SetN$.

A \emph{stopping time} is an $(\SetN \cup \{ \cinfty \})$-valued random variable $T$
such that $\{ v \in \Sigma^\cinfty \mid T(v) = t \} \in \F_t$ for all $t \in \SetN$.
Given a supermartingale $(X_t)_{t \in \SetN}$,
the \emph{stopped process} $(X_{\min\{t, T\}})_{t \in \SetN}$
is a supermartingale~\cite[Thm.\ 5.2.6]{Durrett:10}.
If $(X_t)_{t \in \SetN}$ is bounded,
the limit of the stopped process, $X_T$, exists almost surely
even if $T = \cinfty$ (Martingale Convergence Theorem~\cite[Thm.\ 5.2.8]{Durrett:10}).
We use the following variant on Doob's Optional Stopping Theorem for supermartingales.

\begin{theorem}[{Optional Stopping Theorem~\cite[Thm.\ 5.7.6]{Durrett:10}}]
\label{thm:optional-stopping}
Let $(X_t)_{t \in \SetN}$ be a nonnegative supermartingale
and let $T$ be a stopping time.
The random variable $X_T$ is almost surely well defined and
$\E[X_T] \leq \E[X_0]$.
\end{theorem}

For two probability measures $P$ and $Q$ on $(\Sigma^\cinfty, \Foo)$,
the measure $Q$ is called
\emph{absolutely continuous with respect to $P$ on cylinder sets}
iff $Q(\Gamma_u) = 0$ for all $u \in \Sigma^*$ with $P(\Gamma_u) = 0$.
We exploit the following two theorems that
state the connection between
probability measures on infinite strings and martingales.
For two probability measures $P$ and $Q$
the quotient $Q/P$ is a nonnegative $P$-martingale
if $Q$ is absolutely continuous with respect to $P$ on cylinder sets.
Conversely, for every nonnegative $P$-martingale
there is a probability measure $Q$ on $(\Sigma^\cinfty, \Foo)$ such that
the martingale is $P$-almost surely a multiple of $Q/P$ and
$Q$ is absolutely continuous with respect to $P$ on cylinder sets.

\begin{theorem}[{Measures $\rightarrow$ Martingales~\cite[II§7 Ex.\ 3]{Doob:53}}]
\label{thm:measure-martingale}
Let $Q$ and $P$ be two probability measures on $(\Sigma^\cinfty, \Foo)$ such that
$Q$ is absolutely continuous with respect to $P$ on cylinder sets.
Then the stochastic process $(X_t)_{t \in \SetN}$,
\ifshort
\else \[ X_t(v) := \frac{Q(\Gamma_{v_{1:t}})}{P(\Gamma_{v_{1:t}})} \]
\fi
is a nonnegative $P$-martingale
with $\E[X_t] = 1$.
\end{theorem}

\begin{theorem}[{Martingales $\rightarrow$ Measures}]
\label{thm:martingale-measure}
Let $P$ be a probability measure on $(\Sigma^\cinfty, \Foo)$ and
let $(X_t)_{t \in \SetN}$ be a nonnegative $P$-martingale
with $\E[X_t] = 1$.
There is a probability measure $Q$ on $(\Sigma^\cinfty, \Foo)$
that is absolutely continuous with respect to $P$ on cylinder sets and
for all $v \in \Sigma^\cinfty$ and
all $t \in \SetN$ with $P(\Gamma_{v_{1:t}}) > 0$,
\ifshort
\else \[ X_t(v) = \frac{Q(\Gamma_{v_{1:t}})}{P(\Gamma_{v_{1:t}})}. \]
\fi
\end{theorem}
\ifshort\else
For completeness,
we provide proofs for
\autoref{thm:measure-martingale} and \autoref{thm:martingale-measure} in
Appendix \ref{app:measures-martingales}.
\fi

\thmcountershouldbe{16} 
\begin{remark}[Absolute continuity and absolute continuity on cylinder sets]
\label{rem:absolute-continuity}
A measure $Q$ is called \emph{absolutely continuous with respect to $P$} iff
$Q(A) = 0$ implies $P(A) = 0$ \emph{for all measurable sets $A \in \F_\omega$}.
Absolute continuity trivially implies absolute continuity on cylinder sets.
However, the converse is not true:
absolute continuity on cylinder sets is a strictly weaker condition than
absolute continuity.

Let $P$ be a Bernoulli($2/3$) and $Q$ be a Bernoulli($1/3$) process.
Formally, we fix $\Sigma = \{ 0, 1 \}$ and
define for all $u \in \Sigma^*$,
\begin{align*}
    P(\Gamma_u)
&:= \left( \tfrac{2}{3} \right)^{\mathrm{ones}(u)}
    \left( \tfrac{1}{3} \right)^{\mathrm{zeros}(u)}, \\
    Q(\Gamma_u)
&:= \left( \tfrac{1}{3} \right)^{\mathrm{ones}(u)}
    \left( \tfrac{2}{3} \right)^{\mathrm{zeros}(u)},
\end{align*}
where $\mathrm{ones}(u)$ denotes the number of ones in $u$ and
$\mathrm{zeros}(u)$ denotes the number of zeros in $u$.
Both measures $P$ and $Q$ are nonzero on all cylinder sets:
$Q(\Gamma_u) \geq 3^{-|u|} > 0$ and $P(\Gamma_u) \geq 3^{-|u|} > 0$
for every $u \in \Sigma^*$.
Therefore
$Q$ is absolutely continuous with respect to $P$ \emph{on cylinder sets}.
However, $Q$ is \emph{not} absolutely continuous with respect to $P$:
define
\[
A := \left\{ v \in \Sigma^\omega \mid
        \limsup_{t \to \infty} \tfrac{1}{t} \mathrm{ones}(v_{1:t})
          \leq \tfrac{1}{2}
     \right\}.
\]
The set $A$ is $\F_\omega$-measurable
since $A = \bigcap_{n = 1}^\infty \bigcup_{u \in U_n} \Gamma_u$ with
$U_n := \{ u \in \Sigma^* \mid |u| \geq n \text{ and } \mathrm{ones}(u) \leq |u| / 2 \}$,
the set of all finite strings of length at least $n$
that have at least as many zeros as ones.
We have that $P(A) = 0$ and $Q(A) = 1$,
hence $Q$ is not absolutely continuous with respect to $P$.

While \autoref{thm:measure-martingale} trivially also holds
if $Q$ is absolutely continuous with respect to $P$,
\autoref{thm:martingale-measure} does not imply that
$Q$ is absolutely continuous with respect to $P$.
Consider the process $X_0(v) := 1$,
\[
X_{t+1}(v) :=
\begin{cases}
2 X_t, &\text{if } v_{t+1} = 0 \text{ and} \\
\tfrac{1}{2} X_t, &\text{if } v_{t+1} = 1.
\end{cases}
\]
The process $(X_t)_{t \in \SetN}$ is a nonnegative $P$-martingale since
every $X_t$ is $\F_t$-measurable and for $u = v_{1:t}$ we have
\begin{align*}
   \E[X_{t+1} \mid \F_t ](v)
&= P(\Gamma_{u0} \mid \Gamma_u) 2 X_t(v)
   + P(\Gamma_{u1} \mid \Gamma_u) \tfrac{1}{2} X_t(v) \\
&= \tfrac{1}{3} 2 X_t(v) + \tfrac{2}{3} \cdot \tfrac{1}{2} X_t(v) = X_t(v).
\end{align*}
Moreover,
\begin{align*}
   Q(\Gamma_u)
&= \left( \tfrac{1}{3} \right)^{\mathrm{ones}(u)}
   \left( \tfrac{2}{3} \right)^{\mathrm{zeros}(u)} \\
&= \left( \tfrac{2}{3} \right)^{\mathrm{ones}(u)}
   \left( \tfrac{1}{3} \right)^{\mathrm{zeros}(u)}
   2^{-\mathrm{ones}(u)} 2^{\mathrm{zeros}(u)}
=  P(\Gamma_u) X_t(v).
\end{align*}
Hence
$X_t(v) = Q(\Gamma_{v_{1:t}}) / P(\Gamma_{v_{1:t}})$ $P$-almost surely.
The measure $Q$ is uniquely defined by its values on the cylinder sets,
and as shown above, $Q$ is not absolutely continuous with respect to $P$.
\qed
\end{remark}
\thmcounterbacktonormal

\section{Martingale Upcrossings}
\label{sec:martingale-upcrossings}

Fix $c \in \mathbb{R}$ and $\eps > 0$, and
let $(X_t)_{t \in \SetN}$ be a martingale
over the probability space $(\Sigma^\cinfty, \Foo, P)$.
Let $t_1 < t_2$.
We say the process $(X_t)_{t \in \SetN}$ does an \emph{$\eps$-upcrossing} between $t_1$ and $t_2$ iff
$X_{t_1} \leq c - \eps$ and $X_{t_2} \geq c + \eps$.
Similarly, we say $(X_t)_{t \in \SetN}$ does an \emph{$\eps$-downcrossing} between $t_1$ and $t_2$ iff
$X_{t_1} \geq c + \eps$ and $X_{t_2} \leq c - \eps$.
Except for the first upcrossing, consecutive upcrossings always involve intermediate downcrossings.
Formally, we define the stopping times
\begin{align*}
T_0(v)      &:= 0, \\
T_{2k+1}(v) &:= \inf \{ t > T_{2k}(v)   \mid X_t(v) \leq c - \eps \}, \text{ and} \\
T_{2k+2}(v) &:= \inf \{ t > T_{2k+1}(v) \mid X_t(v) \geq c + \eps \}.
\end{align*}
The $T_{2k}(v)$ denote the indexes of upcrossings.
We count the number of upcrossings
with the random variable $U_t^X(c - \eps, c + \eps)$,
where
\[
U_t^X(c - \eps, c + \eps)(v) := \sup \{ k \geq 0 \mid T_{2k}(v) \leq t \}
\]
and $U^X(c - \eps, c + \eps) := \sup_{t \in \SetN} U_t^X(c - \eps, c + \eps)$
denotes the total number of upcrossings.
We omit the superscript $X$
if the martingale $(X_t)_{t \in \SetN}$ is clear from context.

The following notation is used in the proofs.
Given a monotone decreasing function $f: \SetN \to [0, 1)$ and $m,k \in \SetN$,
we define the event $E_{m,k}^{X,f}$
that there are at least $k$-many $f(m)$-upcrossings:
\[
E_{m,k}^{X,f} := \left\{ v \in \Sigma^\cinfty \mid U^X(1 - f(m), 1 + f(m))(v) \geq k \right\}.
\]
For all $m,k \in \SetN$ we have
$E_{m,k}^{X,f} \supseteq E_{m,k+1}^{X,f}$ and
$E_{m,k}^{X,f} \subseteq E_{m+1,k}^{X,f}$.
Again, we omit $X$ and $f$ in the superscript if they are clear from context.

\section{Indefinitely Oscillating Martingales}
\label{sec:indefinitely-oscillating-martingales}

\newcommand{\firstfigure}{
\centering
\begin{tikzpicture}[scale=\ifshort\else 1.0\fi]
\draw[->] (-0.1,0) -- (8.3,0) node[right] {$t$};
\draw[->] (0,-0.1) -- (0,4) node[above] {$X_t$};

\draw[color=gray,dashed] (8.2,2) -- (-0.1,2) node[left,color=black] {$1$};

\draw[color=orange] (0,3.2) -- (1,3.2) -- (1,2.6) -- (2,2.6) -- (2,2.4)
  -- (3.5, 2.4) -- (3.5, 2.3) -- (5, 2.3) -- (5, 2.24) -- (8.2, 2.24)
  node[above] {$1 + f(M_t)$};
\draw[color=orange] (0,0.8) -- (1,0.8) -- (1,1.4) -- (2,1.4) -- (2,1.6)
  -- (3.5, 1.6) -- (3.5, 1.7) -- (5, 1.7) -- (5, 1.76) -- (8.2, 1.76)
  node[below] {$1 - f(M_t)$};

\draw (0,2) -- (0.5,0.8) -- (1,3.2) -- (1.5, 1.4) -- (2, 2.6) -- (2.5, 3.6)
  -- (3, 1.6) -- (3.5, 2.4) -- (4, 1.7) -- (4.5, 1.1) -- (5, 2.3)
  -- (5.5, 1.76) -- (6, 1.28) -- (6.5, 0.32) -- (7, 0.64) -- (7.5, 0.9)
  -- (8, 0.8) -- (8.2, 0.48);
\draw[fill]
  (0,2) circle (0.05)
  (0.5,0.8) circle (0.05)
  (1,3.2) circle (0.05)
  (1.5,1.4) circle (0.05)
  (2, 2.6) circle (0.05)
  (2.5, 3.6) circle (0.05)
  (3, 1.6) circle (0.05)
  (3.5, 2.4) circle (0.05)
  (4, 1.7) circle (0.05)
  (4.5, 1.1) circle (0.05)
  (5, 2.3) circle (0.05)
  (5.5, 1.76) circle (0.05)
  (6, 1.28) circle (0.05)
  (6.5, 0.32) circle (0.05)
  (7, 0.64) circle (0.05)
  (7.5, 0.9) circle (0.05)
  (8, 0.8) circle (0.05);
\end{tikzpicture}
\caption{
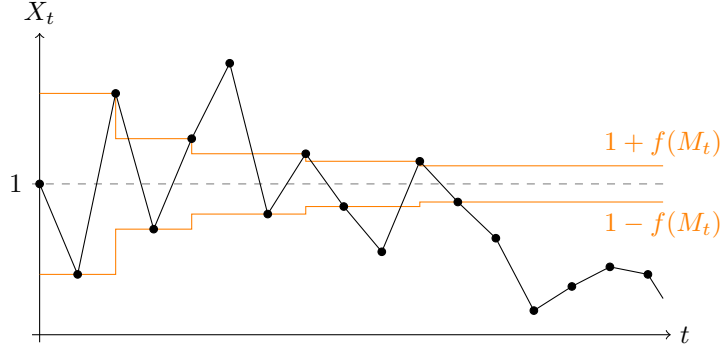
An example evaluation of the martingale defined in the proof of
\autoref{thm:lower-bound}.
}
\label{fig:oscillating-martingale}
}

In this section we construct a class of martingales that has a high probability of doing
an infinite number of upcrossings.
The magnitude of the upcrossings decreases at a rate of a given summable function $f$
(a function $f$ is called \emph{summable} iff
it has finite $L_1$-norm, i.e., $\sum_{i=1}^\infty f(i) < \infty$),
and the value of the martingale $X_t$ oscillates back and forth between $1 - f(M_t)$ and $1 + f(M_t)$,
where $M_t$ denotes the number of upcrossings so far.
The process has a monotone decreasing chance of escaping the oscillation.
We need the following condition on the probability measure $P$.

\thmcountershouldbe{17} 
\begin{definition}[Perpetual Entropy]
\label{def:perpetual-entropy}
A probability measure $P$ has \emph{perpetual entropy}
iff there is an $\eps > 0$ such that
for every $u \in \Sigma^*$ and $v \in \Sigma^\omega$ with $P(\Gamma_u) > 0$
there is an $a \in \Sigma$ and a $t \in \SetN$ with
$1 - \eps > P(\Gamma_{uv_{1:t}a} \mid \Gamma_{uv_{1:t}}) > \eps$.
\end{definition}
\thmcounterbacktonormal

This condition states that after seeing some string $u \in \Sigma^*$,
there is always some future time point where
there are two symbols
that both have conditional probability greater than $\eps$.
In other words,
observing data distributed according to $P$,
we almost surely never run out of symbols with significant entropy.
This is stronger than
demanding that the observed string is nonconstant with high probability,
because we get a single lower bound $\eps$ for \emph{all} observed strings $u$.

\begin{theorem}[An indefinitely oscillating martingale]
\label{thm:lower-bound}
Let $0 < \delta < 1/2$ and
let $f:\SetN \to [0, 1)$ be any monotone decreasing function such that
$\sum_{i=1}^\infty f(i) \leq \delta / 2$.
For every probability measure $P$ with perpetual entropy
there is a nonnegative martingale $(X_t)_{t \in \SetN}$
with $\E[X_t] = 1$ and
\[
     P [ \forall m.\; U(1 - f(m), 1 + f(m)) \geq m ]
\geq 1 - \delta.
\]
\end{theorem}
\begin{proof}
By grouping symbols from $\Sigma$ into two groups,
we can without loss of generality assume that $\Sigma = \{ 0, 1 \}$.
Since $P(\Gamma_{u0} \mid \Gamma_u) + P(\Gamma_{u1} \mid \Gamma_u) = 1$,
we can define a function $a:\Sigma^* \to \Sigma$ that assigns
to every string $u \in \Sigma^*$ a symbol $a_u := a(u)$ such that
$p_u := P(\Gamma_{ua_u} \mid \Gamma_u) \leq \frac{1}{2}$.
In \autoref{claim:perpetual-entropy} we show that without loss of generality,
we can group such that $p_u > \eps$ infinitely often for some $\eps > 0$.

In the following we define the stochastic process $(X_t)_{t \in \SetN}$.
This process depends on the random variables $M_t$ and $\gamma_t$,
which are defined below.
Let $v \in \Sigma^\omega$ and $t \in \SetN$ be given
and define $u := v_{1:t}$.
For $t = 0$, we set $X_0(v) := 1$,
and if $p_u = 0$, we set $X_{t+1} = X_t$.
Otherwise we distinguish the following three cases.
\begin{enumerate}[(i)]
\item For $X_t(v) \geq 1$:
\[
X_{t+1}(v) :=
\begin{cases}
1 - f(M_t(v))
  &\text{if } v_{t+1} \neq a_u, \\
X_t(v) + \frac{1 - p_u}{p_u} (X_t(v) - (1 - f(M_t(v))))
  &\text{if } v_{t+1} = a_u.
\end{cases}
\]
\item For $1 > X_t(v) \geq \gamma_t(v)$:
\[
X_{t+1}(v) :=
\begin{cases}
X_t(v) - \gamma_t(v)
  &\text{if } v_{t+1} \neq a_u, \\
1 + f(M_t(v))
  &\text{if } v_{t+1} = a_u.
\end{cases}
\]
\item For $X_t(v) < \gamma_t(v)$ and $X_t(v) < 1$: \\
let $d_t(v) := \min \{ \frac{p_u}{1 - p_u} X_t(v),
        \tfrac{1 - p_u}{p_u} \gamma_t(v) - 2f(M_t(v))
     \}$;
\[
X_{t+1}(v) :=
\begin{cases}
X_t(v) + d_t(v) 
  &\text{if } v_{t+1} \neq a_u, \\
X_t(v) - \tfrac{1 - p_u}{p_u} d_t(v)
  &\text{if } v_{t+1} = a_u.
\end{cases}
\]
\end{enumerate}
The random variables $M_t$ and $\gamma_t$ are defined as
\begin{align*}
    \gamma_t(v)
&:= \tfrac{p_u}{1 - p_u} \Big( 1 + f(M_t(v)) - X_t(v) \Big) \\
    M_t(v)
&:= 1 + \argmax_{m \in \SetN}
      \left\{ \forall k \leq m.\; U_t^X(1 - f(k), 1 + f(k)) \geq k \right\},
\end{align*}
i.e.,
$M_t$ is $1$ plus
the number of upcrossings completed  up to time $t$.

\ifshort\else
\begin{figure}[t]
\firstfigure
\end{figure}
\fi

We give an intuition for the behavior of the process $(X_t)_{t \in \SetN}$.
For all $m$, the following repeats.
First $X_t$ increases while reading $a_u$'s
until it reads one symbol that is not $a_u$ and then jumps down to $1 - f(m)$.
Subsequently, $X_t$ decreases while not reading $a_u$'s
until it falls below $\gamma_t$ or reads an $a_u$ and then jumps up to $1 + f(m)$.
If it falls below $1$ and $\gamma_t$, then at every step,
it can either jump up to $1 - f(m)$ or jump down to $0$,
whichever one is closest
(the distance to the closest of the two is given by $d_t$).
See \autoref{fig:oscillating-martingale} for a visualization.

For notational convenience,
in the following we omit writing the argument $v$ to the random variables
$X_t$, $\gamma_t$, $M_t$, and $d_t$.

\begin{claims}
\newclaim{$(X_t)_{t \in \SetN}$ is a martingale.}
\label{claim:is-martingale}
Each $X_{t+1}$ is $\F_{t+1}$-measurable,
since it uses only the first $t + 1$ symbols of $v$.
Writing out cases (i), (ii), and (iii), we get
\begin{align*}
   \E[X_{t+1} \mid \F_t]
&\stackrel{(i)}{=} (1 - f(M_t)) (1 - p_u)
   + \big(X_t + \tfrac{1 - p_u}{p_u}(X_t - (1 - f(M_t))) \big) p_u
 = X_t, \\
   \E[X_{t+1} \mid \F_t]
&\stackrel{(ii)}{=} \big(X_t - \tfrac{p_u}{1 - p_u}((1 + f(M_t)) - X_t) \big) (1 - p_u)
   + (1 + f(M_t)) p_u
 = X_t, \\
   \E[X_{t+1} \mid \F_t]
&\stackrel{(iii)}{=} (X_t + d_t) (1 - p_u)
   + (X_t - \tfrac{1 - p_u}{p_u} d_t) p_u
 = X_t.
\end{align*}

\newclaim{If $X_t \geq 1 - f(M_t)$ then $X_t > \gamma_t$.}
\label{claim:gamma-is-small}
In this case
\[
     \gamma_t
=    \tfrac{p_u}{1 - p_u} (1 + f(M_t) - X_t)
\leq 2 \tfrac{p_u}{1 - p_u} f(M_t),
\]
and thus with $p_u \leq \frac{1}{2}$ and
$f(M_t) \leq \sum_{k=1}^\infty f(k) \leq \frac{\delta}{2} < \frac{1}{4} < \frac{1}{3}$,
\[
      X_t - \gamma_t
\geq 1 - f(M_t) - 2 \tfrac{p_u}{1 - p_u} f(M_t)
=    1 - \tfrac{1 + p_u}{1 - p_u} f(M_t)
\geq 1 - 3 f(M_t)
 >    0.
\]

\newclaim{If $p_u > 0$, $X_t < \gamma_t$, and $X_t < 1$ then $d_t \geq 0$.}
\label{claim:d-nonnegative}
We have $\frac{p_u}{1 - p_u} X_t \geq 0$ since $p_u > 0$ and $X_t \geq 0$.
Moreover, $\frac{1 - p_u}{p_u} \gamma_t - 2f(M_t) = 1 - f(M_t) - X_t > 0$
by the contrapositive of \autoref{claim:gamma-is-small}.

\newclaim{The following holds for cases (i), (ii), and (iii).
\label{claim:bounds-on-X_t+1}
\begin{enumerate}[(a)]
\item In case (i):
	$X_{t+1} \geq X_t$ or $X_{t+1} = 1 - f(M_t)$.
\item In case (ii):
	$X_{t+1} \leq X_t$ or $X_{t+1} = 1 + f(M_t)$.
\item In case (iii):
	$X_t < 1 - f(M_t)$ and $X_{t+1} \leq 1 - f(M_t)$.
\end{enumerate}}
If $p_u = 0$ then $X_{t+1} = X_t$, so (a) and (b) hold trivially.
Otherwise, for (a) we have $\frac{1 - p_u}{p_u} > 0$ and
$X_t \geq 1 - f(M_t)$.
For (b) we have $\gamma_t > 0$ since $X_t < 1 + f(M_t)$.
For (c), $X_t < 1 - f(M_t)$ follows from
the contrapositive of \autoref{claim:gamma-is-small}.
If $p_u > 0$ then by \autoref{claim:d-nonnegative}
we have $d_t \geq 0$ and hence
$X_{t+1} \leq X_t + d_t \leq X_t + (1 + f(M_t) - X_t) - 2f(M_t) = 1 - f(M_t)$.

\newclaim{$X_t \geq 0$ and $\E[X_t] = 1$.}
\label{claim:nonnegative-and-expectation}
The latter follows from
\ifshort
\else
\[
    \E[X_t]
~=~ \E[\E[X_t \mid \F_{t-1}]]
~=~ \E[X_{t-1}]
~=~ \ldots
~=~ \E[X_0]
~=~ 1.
\]
\fi
Regarding the former,
we use $0 \leq f(M_t) < 1$ to conclude
\begin{itemize}
\setlength{\itemindent}{1.5em}
\item[(i$\neq$)] $1 - f(M_t) \geq 0$,
\item[(i$=$)] $\frac{1 - p_u}{p_u} (X_t - (1 - f(M_t))) \geq 0$ for $X_t \geq 1$,
\item[(ii$\neq$)] $X_t - \gamma_t \geq 0$ for $X_t \geq \gamma_t$,
\item[(ii$=$)] $1 + f(M_t) \geq 0$,
\item[(iii$\neq$)] $X_t + d_t \geq 0$ since $d_t \geq 0$
	by \autoref{claim:d-nonnegative}, and
\item[(iii$=$)] $X_t - \tfrac{1 - p_u}{p_u} d_t \geq 0$ since $d_t \leq \tfrac{p_u}{1 - p_u} X_t$.
\end{itemize}

\newclaim{$X_t \leq 1 - f(M_t)$ or $X_t \geq 1 + f(M_t)$ for all $t \geq T_1$.}
\label{claim:X_t-jumps}
We use induction on $t$:
the induction start holds with $X_{T_1} \leq 1 - f(M_t)$ and
the induction step follows from \autoref{claim:bounds-on-X_t+1}.

\newclaim{$P( \{ v \in \Sigma^\omega \mid p_{v_{1:t}} > \eps \text{ for infinitely many } t \} ) = 1$
for some $\eps > 0$.}
\label{claim:perpetual-entropy}
By assumption $P$ has perpetual entropy;
let $\eps$ be as in \autoref{def:perpetual-entropy}.
\[
   A
:= \{ v \in \Sigma^\omega \mid P(\Gamma_{v_{1:t}}) > 0 \text{ for all } t \}
 \]
Its complement $A^c = \bigcup_{u \in \Sigma^*: P(\Gamma_u) = 0} \Gamma_u$
is the countable union of null sets and therefore $P(A) = 1$.
Let $v \in A$ be some outcome,
let $t \in \SetN$ be the current time step,
and define $u := v_{1:t}$.
Because $P$ has perpetual entropy and $P(\Gamma_u) > 0$ since $v \in A$,
there exists $u' \in \Sigma^*$, $a \in \Sigma$, and $v' \in \Sigma^\omega$
such that $v = uu'av'$ and
$1 - \eps > P(\Gamma_{uu'a} \mid \Gamma_{uu'}) > \eps$.
If $P(\Gamma_{uu'a} \mid \Gamma_{uu'}) \leq 1/2$ we can select $a_{uu'} := a$;
if $P(\Gamma_{uu'a} \mid \Gamma_{uu'}) > 1/2$
then,
with abuse of notation,
for the symbol group $b := \Sigma \setminus \{ a \}$
we have $\eps < P(\Gamma_{uu'b} \mid \Gamma_{uu'}) \leq 1/2$ and hence we
can select $a_{uu'} := b$.
In either case $p_{uu'} > \eps$ for a suitable grouping of symbols.

\newclaim{$(X_t)_{t \in \SetN}$ converges almost surely to a random variable
$X_\omega \in \{ 0, 1 \}$.}
\label{claim:X-does-not-converge}
According to the Martingale Convergence Theorem~\cite[Thm.\ 5.2.8]{Durrett:10},
the process $(X_t)_{t \in \SetN}$ converges almost surely
to a random variable $X_\omega$.
Assume that $X_\omega$ attains some value $x_\omega$ other than $0$ and $1$.
Pick an $\eps' > 0$ such that
$|x_\omega| > 2\eps'$ and $|1 - x_\omega| > 2\eps'$.
Since $X_t \to x_\omega$
we have $|x_\omega - X_t| < \eps'$ for all but finitely many $t$, and hence
there is a $t_0 \in \SetN$ such that
$|X_t| > \eps'$ and $|1 - X_t| > \eps'$ for all $t \geq t_0$.
Recall that $\eps > 0$ is fixed and depends only on $P$.
Below we show for cases (i), (ii), and (iii) that
$|X_{t+1} - X_t| > \min\{\eps \cdot \eps', \eps', \frac{1}{8} \}$
if $p_u > \eps$.
By \autoref{claim:perpetual-entropy}
we almost surely have infinitely many $t \geq t_0$ with $p_u > \eps$,
which is a contradiction to the fact that
$(X_t)_{t \in \SetN}$ converges almost surely.
\begin{enumerate}[(i)]
\item Assume $X_t \geq 1$, then $X_t > 1 + \eps'$ by assumption.
	Either $X_{t+1} = 1 - f(M_t) \leq 1 < X_t - \eps'$ or
	$X_{t+1} = X_t + \frac{1 - p_u}{p_u} (X_t - 1 + f(M_t))
	> X_t + (X_t - 1 + f(M_t)) \geq X_t + (X_t - 1) > X_t + \eps'$
	because $p_u \leq \frac{1}{2}$ implies $\frac{1 - p_u}{p_u} \geq 1$.

\item Assume $\gamma_t \leq X_t < 1$, then $\eps' < X_t < 1 - \eps'$.
	Either $X_{t+1} = 1 + f(M_t) \geq 1 > X_t + \eps'$ or
	$X_{t+1} = X_t - \gamma_t$ and thus
	$X_t - X_{t+1} = \gamma_t = \frac{p_u}{1 - p_u}(1 + f(M_t) - X_t)
	> \eps (1 + f(M_t) - X_t) \geq \eps (1 - X_t) > \eps \eps'$.

\item Assume $X_t < \gamma_t$ and $X_t < 1$,
	then since $0 \leq X_t$ by \autoref{claim:nonnegative-and-expectation},
	$\eps' < X_t < \gamma_t$ and $X_t < 1 - \eps'$.
	Either $d_t = \frac{p_u}{1 - p_u} X_t > \eps \eps'$ and we are done,
	or $d_t = \frac{1 - p_u}{p_u} \gamma_t - 2f(M_t)$.
	If $X_t \geq \frac{5}{8}$ then
	$d_t > \frac{1 - p_u}{p_u} X_t - 2f(M_t) > X_t - \frac{1}{2} \geq \frac{1}{8}$,
	since $f(M_t) \leq \frac{\delta}{2} < \frac{1}{4}$.
	If $X_t < \frac{5}{8}$ then
	$d_t = 1 - f(M_t) - X_t > \frac{3}{4} - X_t > \frac{1}{8}$.
	Hence either $X_{t+1} - X_t = d_t > \min\{ \eps \eps', \frac{1}{8} \}$
	or $X_t - X_{t+1} = \frac{1 - p_u}{p_u} d_t > d_t
	> \min\{ \eps \eps', \frac{1}{8} \}$.
\end{enumerate}

\newclaim{For all $m \in \SetN$,
if $E_{m,m-1} \neq \emptyset$ then
$P(E_{m,m} \mid E_{m,m-1}) \geq 1 - 2f(m)$.}
\label{claim:single-upcrossing-bound}
Let $v \in E_{m,m-1}$
and let $t_0 \in \SetN$ be a time step such that
exactly $m - 1$ upcrossings have been completed up to time $t_0$, i.e.,
$M_{t_0}(v) = m$.
The subsequent downcrossing is completed eventually with probability $1$:
we are in case (i) and
in every step there is a chance of $1 - p_u \geq \frac{1}{2}$ of completing the downcrossing.
Therefore we assume without loss of generality that the downcrossing has been completed,
i.e., that $t_0$ is such that $X_{t_0}(v) = 1 - f(m)$.
We will bound the probability $p := P(E_{m,m} \mid E_{m,m-1})$
that $X_t$ rises above $1 + f(m)$ after $t_0$ to complete the $m$-th upcrossing.

Define the stopping time $T: \Sigma^\cinfty \to \SetN \cup \{ \cinfty \}$,
\[
T(v) := \inf \{ t \geq t_0 \mid X_t(v) \geq 1 + f(m) \;\lor\; X_t(v) = 0 \},
\]
and define the stochastic process
$Y_t = 1 + f(m) - X_{\min\{t_0 + t, T\}}$.
Because $(X_{\min\{t_0 + t, T\}})_{t \in \SetN}$ is martingale,
$(Y_t)_{t \in \SetN}$ is martingale.
By definition, $X_t$ always stops at $1 + f(m)$ before exceeding it,
thus $X_T \leq 1 + f(m)$,
and hence $(Y_t)_{t \in \SetN}$ is nonnegative.
The \hyperref[thm:optional-stopping]{Optional Stopping Theorem} yields
$\E[Y_{T - t_0} \mid \F_{t_0}] \leq \E[Y_0 \mid \F_{t_0}]$ and thus
$\E[X_T \mid \F_{t_0}] \geq \E[X_{t_0} \mid \F_{t_0}] = 1 - f(m)$.
We show that $X_T \in \{ 0, 1 + f(m) \}$ almost surely.
If $T$ is finite then this holds by definition of $T$.
If $T = \omega$ then the random variable $X_T$ is defined as
the limit $\lim_{t \to \infty} X_t$.
By \autoref{claim:X-does-not-converge} the limit $X_T \in \{ 0, 1 \}$
and according to \autoref{claim:X_t-jumps} we have
$X_t \leq 1 - f(M_t)$ for all $t \in \SetN$,
so $X_t$ cannot converge to $1$.
We conclude that
\[
1 - f(m) \leq \E[X_T \mid \F_{t_0}] = (1 + f(m)) \cdot p + 0 \cdot (1 - p),
\]
hence
$
     P(E_{m,m} \mid E_{m,m-1})
=    p
\geq 1 - f(m) (1 + p)
\geq 1 - 2f(m)
$.

\newclaim{$E_{m+1,m} = E_{m,m}$ and $E_{m+1,m+1} \subseteq E_{m,m}$.}
\label{claim:E-increasing}
By definition of $M_t$,
the $i$-th upcrossings of the process $(X_t)_{t \in \SetN}$
is between $1 - f(i)$ and $1 + f(i)$.
The function $f$ is monotone decreasing,
and by \autoref{claim:X_t-jumps} the process $(X_t)_{t \in \SetN}$
does not assume values between $1 - f(i)$ and $1 + f(i)$.
Therefore the first $m$ $f(m+1)$-upcrossings are also $f(m)$-upcrossings,
i.e., $E_{m+1,m} \subseteq E_{m,m}$.
By definition of $E_{m,k}$ we have
$E_{m+1,m} \supseteq E_{m,m}$ and
$E_{m+1,m+1} \subseteq E_{m+1,m}$.

\newclaim{$P(E_{m,m}) \geq 1 - \sum_{i=1}^m 2f(i)$.}
\label{claim:bound-on-E}
For $P(E_{0,0}) = 1$ this holds trivially.
Using \autoref{claim:single-upcrossing-bound} and \autoref{claim:E-increasing}
we conclude inductively
\begin{align*}
      P(E_{m,m})
&=    P(E_{m,m} \cap E_{m,m-1})
 =    P(E_{m,m} \mid E_{m,m-1}) P(E_{m,m-1}) \\
&=    P(E_{m,m} \mid E_{m,m-1}) P(E_{m-1,m-1}) \\
&\geq (1 - 2f(m)) \left( 1 - \sum_{i=1}^{m-1} 2f(i) \right)
 \geq 1 - \sum_{i=1}^m 2f(i).
\end{align*}
\end{claims}

From \autoref{claim:E-increasing} follows
$\bigcap_{i=1}^m E_{i,i} = E_{m,m}$ and therefore
$P(\bigcap_{i=1}^\infty E_{i,i})
= \lim_{m \to \infty} P(E_{m,m})
\geq 1 - \sum_{i=1}^\infty 2f(i)
\geq 1 - \delta$.
\end{proof}

\autoref{thm:lower-bound} gives a \emph{uniform} lower bound on the probability
for many upcrossings:
it states the probability of the event that \emph{for all $m \in \SetN$},
$U(1 - f(m), 1 + f(m)) \geq m$ holds.
This is a lot stronger than the nonuniform bound
$P[U(1 - f(m), 1 + f(m)) \geq m] \geq 1 - \delta$ for all $m \in \SetN$:
the quantifier is inside the probability statement.

As an immediate consequence of \autoref{thm:lower-bound},
we get the following uniform lower bound on the \emph{expected} number of upcrossings.

\begin{corollary}[Expected Upcrossings]
\label{cor:lower-bound}
\ifshort
\else
Let $0 < \delta < 1/2$ and
let $f:\SetN \to [0, 1)$ be any monotone decreasing function such that
$\sum_{i=1}^\infty f(i) \leq \delta / 2$.
For every probability measure $P$ with perpetual entropy
there is a nonnegative martingale $(X_t)_{t \in \SetN}$
with $\E[X_t] = 1$ and
\fi
for all $m \in \SetN$,
\[
\E[U(1 - f(m), 1 + f(m))] \geq m (1 - \delta).
\]
\end{corollary}
\begin{proof}
From \autoref{thm:lower-bound} and  Markov's inequality.
\end{proof}

\ifshort
\else
By choosing a specific slowly decreasing but summable function $f$,
\fi
we get the following concrete results.

\begin{corollary}[Concrete lower bound]
\label{cor:lower-bound-concrete}
Let $0 < \delta < 1 / 2$.
For every probability measure $P$ with perpetual entropy
there is a nonnegative martingale $(X_t)_{t \in \SetN}$
with $\E[X_t] = 1$ such that
\begin{gather*}
P \left[ \forall \eps > 0.\; U(1 - \eps, 1 + \eps)
  \in \Omega \left( \tfrac{\delta}{\eps \left( \ln \frac{1}{\eps} \right)^2} \right) \right]
\geq 1 - \delta
\text{ and } \\
     \E[U(1 - \eps, 1 + \eps)]
\in \Omega \Big( \tfrac{1}{\eps \left( \ln \frac{1}{\eps} \right)^2} \Big).
\end{gather*}
Moreover, for all $\eps < 0.015$ we get
$\E[U(1 - \eps, 1 + \eps)] > \tfrac{\delta (1 - \delta)}{\eps \left( \ln \frac{1}{\eps} \right)^2}$ and
\[
P \left[ \forall \eps < 0.015.\; U(1 - \eps, 1 + \eps)
  > \tfrac{\delta}{\eps \left( \ln \frac{1}{\eps} \right)^2} \right]
\geq 1 - \delta.
\]
\end{corollary}
\ifshort\else
\begin{proof}
Define
\[
  g:(0,e^{-2}] \to [0, \infty),
  \quad\quad
  \eps \mapsto 2\delta \left( \frac{1}{\eps (\ln \eps)^2} - \frac{e^2}{4} \right).
\]
We have $g(e^{-2}) = 0$, $\lim_{\eps \to 0} g(\eps) = \infty$, and
\[
  \frac{dg}{d\eps}(\eps)
~=~ 2\delta \left(
    \frac{-1}{\eps^2 (\ln \eps)^2}
    + \frac{-2}{\eps^2 (\ln \eps)^3}
  \right) ~=~
  -\frac{2\delta(2+\ln\eps)}{\eps^2(\ln \eps)^3}
  ~<~ 0 \text{ on } (0, e^{-2}).
\]
Therefore the function $g$ is strictly monotone decreasing and hence invertible.
Choose $f := g^{-1}$.
Using the substitution $t = g(\eps)$, $dt = \frac{dg}{d\eps}(\eps)d\eps$,
\begin{align*}
      \sum_{t=1}^\infty f(t)
&\leq \int_0^\infty f(t) dt
 =    \int_{g^{-1}(0)}^{g^{-1}(\infty)} f(g(\eps)) \frac{dg}{d\eps}(\eps) d\eps \\
&=    2\delta \left(
        \int_{e^{-2}}^0 \frac{-1}{\eps (\ln \eps)^2} d\eps
        + \int_{e^{-2}}^0 \frac{-2}{\eps (\ln \eps)^3} d\eps
      \right) \\
&=    2\delta \left(
        \left[ \tfrac{1}{\ln \eps} \right]_{e^{-2}}^0
        + \left[ \tfrac{1}{(\ln \eps)^2} \right]_{e^{-2}}^0
      \right)
 =    2\delta \left( \tfrac{1}{2} - \tfrac{1}{4} \right)
 =    \tfrac{\delta}{2}.
\end{align*}
Now we apply \autoref{thm:lower-bound} and \autoref{cor:lower-bound} to $m := g(\eps)$ and get
\begin{align*}
      P \left[ U(1 - \eps, 1 + \eps)
         \geq 2\delta \left( \tfrac{1}{\eps (\ln \eps)^2} - \tfrac{e^2}{4} \right) \right]
&\geq 1 - \delta, \text{ and} \\
      \E[U(1 - \eps, 1 + \eps)]
&\geq 2 \delta (1 - \delta)
        \left( \tfrac{1}{\eps (\ln \eps)^2} - \tfrac{e^2}{4} \right).
\end{align*}
For $\eps < 0.015$, we have $\frac{1}{\eps \left( \ln \eps \right)^2} > \frac{e^2}{2}$, hence
$g(\eps) > \frac{\delta}{\eps (\ln \eps)^2}$.
\end{proof}\fi

The concrete bounds given in \autoref{cor:lower-bound-concrete}
are \emph{not} the asymptotically optimal ones:
there are summable functions that decrease even more slowly.
For example,
we could multiply \ifshort\else the function $g$ \fi with the factor
$\sqrt{\ln(1/\eps)}$ (which still is not optimal).

\section{Martingale Upper Bounds}
\label{sec:upper-bounds}

In this section we state upper bounds on the probability and expectations of many upcrossings
(\hyperref[thm:Dubins-inequality]{Dubins' Inequality} and
\hyperref[thm:upcrossing-inequality]{Doob's Upcrossing Inequality}).
We use the construction from the previous section
to show that these bounds are asymptotically tight.
Moreover, with the following theorem we show that
the uniform lower bound on the probability of many upcrossings guaranteed in \autoref{thm:lower-bound}
is also asymptotically tight.

Every function $f$ is either summable or not.
If $f$ is summable, then we can scale it with a constant factor such that
its sum is smaller than $\frac{\delta}{2}$,
and then apply the construction of \autoref{thm:lower-bound}.
If $f$ is not summable,
the following theorem implies that
there is no \emph{uniform} lower bound on
the probability of having at least $m$-many $f(m)$-upcrossings.

\begin{theorem}[Upper bound on upcrossing rate]
\label{thm:upper-bound}
Let $f: \SetN \to [0, 1)$ be a monotone decreasing function such that
$\sum_{t=1}^\infty f(t) = \infty$.
For every probability measure $P$ and
for every nonnegative $P$-martingale $(X_t)_{t \in \SetN}$
with $\E[X_t] = 1$,
\[
P[ \forall m.\; U(1 - f(m), 1 + f(m)) \geq m] = 0.
\]
\end{theorem}
\begin{proof}
Define the events $D_m := \bigcup_{i=1}^m E_{i,i}^c = 
\{ \forall i \leq m.\; U(1 - f(i), 1 + f(i)) \geq i \}$.
Then $D_m \subseteq D_{m+1}$.
Assume there is a constant $c > 0$ such that
$c \leq P(D_m^c) = P(\bigcap_{i=1}^m E_{i,i})$ for all $m$.
Let $m \in \SetN$, $v \in D_m^c$, and pick $t_0 \in \SetN$
such that the process $X_0(v), \ldots, X_{t_0}(v)$
has completed $i$-many $f(i)$-upcrossings for all $i \leq m$ and
$X_{t_0}(v) \leq 1 - f(m+1)$.
If $X_t(v) \geq 1 + f(m+1)$ for some $t \geq t_0$,
the $(m + 1)$-st upcrossing for $f(m + 1)$ is completed
and thus $v \in E_{m+1, m+1}$.
Define the stopping time $T: \Sigma^\cinfty \to (\SetN \cup \{ \cinfty \})$,
\[
T(v) := \inf \{ t \geq t_0 \mid X_t(v) \geq 1 + f(m+1) \}.
\]
According to the \hyperref[thm:optional-stopping]{Optional Stopping Theorem}
applied to the process $(X_t)_{t \geq t_0}$,
the random variable $X_T$ is almost surely well-defined and
$\E[X_T \mid \F_{t_0}] \leq \E[X_{t_0} \mid \F_{t_0}] = X_{t_0}$.
This yields
$1 - f(m+1) \geq X_{t_0} \geq \E[X_T \mid \F_{t_0} ]$ and
by taking the expectation $\E[ \;\cdot \mid X_{t_0} \leq 1 - f(m+1)]$ on both sides,
\begin{align*}
      1 - f(m+1)
&\geq \E[ X_T \mid X_{t_0} \leq 1 - f(m+1) ] \\
&\geq (1 + f(m+1)) P[X_T \geq 1 + f(m+1) \mid X_{t_0} \leq 1 - f(m+1) ]
\end{align*}
by Markov's inequality.
Therefore
\begin{align*}
      P( E_{m+1,m+1} \mid D_m^c )
=    \;&P[ X_T \geq 1 + f(m+1) \mid X_{t_0} \leq 1 - f(m+1)] \\
     \;&\cdot P[ X_{t_0} \leq 1 - f(m+1) \mid D_m^c] \\
\leq \;&P[X_T \geq 1 + f(m+1) \mid X_{t_0} \leq 1 - f(m+1)] \\
\leq \;&\tfrac{1 - f(m+1)}{1 + f(m+1)} \leq 1 - f(m+1).
\end{align*}
Together with $c \leq P(D_m^c)$ we get
\begin{align*}
     P \left( D_{m+1} \setminus D_m \right)
&=    P \left( E_{m+1,m+1}^c \cap D_m^c \right) \\
&=    P \left( E_{m+1,m+1}^c \mid D_m^c \right) P \left( D_m^c \right)
\geq f(m+1) c.
\end{align*}
This is a contradiction because $\sum_{i=1}^\infty f(i) = \infty$:
\[
      1
\geq  P(D_{m+1})
 =    P \left( \biguplus_{i=1}^m (D_{i+1} \setminus D_i) \right) \\
 =    \sum_{i=1}^m P(D_{i+1} \setminus D_i)
 \geq \sum_{i=1}^m f(i+1) c
 \to  \infty.
\]
Therefore the assumption $P(D_m^c) \geq c$ for all $m$ is false,
and hence we get
$
  P[ \forall m.\; U(1 - f(m), 1 + f(m)) \geq m]
= P(\bigcap_{i=1}^\infty E_{i,i})
= \lim_{m \to \infty} P(D_m^c)
= 0
$.
\end{proof}

\ifshort
\else
By choosing a specific decreasing non-summable function $f$
\fi
for \autoref{thm:upper-bound},
we get that
$U(1-\eps, 1 + \eps) \notin \Omega(\frac{1}{\eps \log(1/\eps)})$
$P$-almost surely.

\begin{corollary}[Concrete upper bound]
\label{cor:upper-bound-concrete}
Let $P$ be a probability measure and
let $(X_t)_{t \in \SetN}$ be a nonnegative martingale
with $\E[X_t] = 1$.
Then for all $a, b > 0$,
\[
P \left[ \forall \eps > 0.\; U(1 - \eps, 1 + \eps) \geq \tfrac{a}{\eps \log(1/\eps)} - b \right] = 0.
\]
\end{corollary}
\ifshort\else
\begin{proof}
We proceed analogously to the proof of \autoref{cor:lower-bound}.
Define
\[
  g: (0, c] \to [g(c), \infty), ~~~~~~ \eps \mapsto \frac{a}{\eps \ln \frac{1}{\eps}} - b
\]
with $c < 1$ and $g(c) \geq 1$.
We have $\lim_{\eps \to 0} g(\eps) \to \infty$ and
\[
  \frac{dg}{d\eps}(\eps)
= \frac{-a}{\eps^2 \ln \frac{1}{\eps}}
  + \frac{-a}{\eps^2 (\ln \frac{1}{\eps})^2}
< 0 \text{ on } (0, c].
\]
Therefore the function $g$ is strictly monotone decreasing and hence invertible.
Choose $f := g^{-1}$.
Using the substitution $t = g(\eps)$,
$dt = \frac{dg}{d\eps}(\eps) d\eps$,
\begin{align*}
      \sum_{t=1}^\infty f(t)
&\geq \int_{g(c)}^\infty f(t) dt
 =    \int_c^{g^{-1}(\infty)} f(g(\eps)) \frac{dg}{d\eps}(\eps) d\eps \\
&=    \int_c^0 \frac{-a}{\eps \ln \frac{1}{\eps}} d\eps
        + \int_c^0 \frac{-a}{\eps (\ln \frac{1}{\eps})^2} d\eps
 =    \int_{-\ln c}^{-\ln 0} \frac{a}{u} du
        + \int_c^0 \frac{-a}{\eps (\ln \frac{1}{\eps})^2} d\eps \\
&=    \left[ a \ln u \right]_{-\ln c}^{+\infty}
        + \left[ \tfrac{a}{\ln \frac{1}{\eps}} \right]_c^0
 =    \infty - a \ln (-\ln c) + 0 - \tfrac{a}{\ln \frac{1}{c}}
 =    \infty.
\end{align*}
Now we apply \autoref{thm:upper-bound}
to $m := g(\eps)$.
\end{proof}\fi

\begin{theorem}[{Dubins' Inequality~\cite[Thm.\ 13.1]{Dubins:62}}]
\label{thm:Dubins-inequality}
For every nonnegative $P$-martingale $(X_t)_{t \in \SetN}$
and for every $c > 0$ and every $\eps > 0$,
\[
     P[ U(c - \eps, c + \eps) \geq k ]
\leq \left( \tfrac{c - \eps}{c + \eps} \right)^k
       \E \left[ \min \left\{\tfrac{X_0}{c - \eps}, 1 \right\} \right].
\]
\end{theorem}

Dubins' Inequality immediately yields
the following bound on the probability of the number of upcrossings.
\[
     P[U(1 - f(m), 1 + f(m)) \geq k]
\leq \left( \tfrac{1 - f(m)}{1 + f(m)} \right)^k.
\]
The construction from \autoref{thm:lower-bound} shows that
this bound is asymptotically tight for $m = k \to \infty$ and $\delta \to 0$:
define the monotone decreasing function $f: \SetN \to [0, 1)$,
\begin{align}
\begin{aligned}
f(t) :=
\begin{cases}
\frac{\delta}{2m}, &\text{if } t \leq m, \text{ and} \\
0, &\text{otherwise}.
\end{cases}
\end{aligned}
\label{eq:finite-f}
\end{align}
Then the martingale from \autoref{thm:lower-bound} yields the lower bound
\[
     P[U(1 - \tfrac{\delta}{2k}, 1 + \tfrac{\delta}{2k}) \geq k]
\geq 1 - \delta,
\]
while \hyperref[thm:Dubins-inequality]{Dubins' Inequality} gives the upper bound
\[
     P[U(1 - \tfrac{\delta}{2k}, 1 + \tfrac{\delta}{2k}) \geq k]
\leq \left( \frac{1 - \frac{\delta}{2k}}{1 + \frac{\delta}{2k}} \right)^k
=    \left( 1 - \frac{2\delta}{2k + \delta} \right)^k
\xrightarrow{k \to \infty} \exp(-\delta).
\]
As $\delta$ approaches $0$,
the value of $\exp(-\delta)$ approaches $1 - \delta$
(but exceeds it since $\exp$ is convex).
For $\delta = 0.2$ and $m = k = 3$,
the difference between the two bounds is already lower than $0.021$.

The following theorem
places an upper bound on the rate of \emph{expected} upcrossings.
\ifshort\else
In Appendix \ref{app:tightness}
we discuss different versions of this inequality
and prove this inequality tight.
\fi

\begin{theorem}[{Doob's Upcrossing Inequality~\cite{Xu12}}]
\label{thm:upcrossing-inequality}
Let $(X_t)_{t \in \SetN}$ be a submartingale.
For every $c \in \mathbb{R}$ and $\eps > 0$,
\[
     \E[U_t(c - \eps, c + \eps)]
\leq \tfrac{1}{2\eps} \E[\max\{ c - \eps - X_t, 0 \}].
\]
\end{theorem}
%
Asymptotically, Doob's Upcrossing Inequality states that with $\eps \to 0$,
\[
\E[U(1 - \eps, 1 + \eps)]
\in O \left(\tfrac{1}{\eps} \right).
\]
Again, we can use the construction of \autoref{thm:lower-bound}
to show that these asymptotics are tight:
Let $f$ be as in \eqref{eq:finite-f}.
Then for $\delta = \frac{1}{2}$,
\autoref{cor:lower-bound} yields a martingale fulfilling the lower bound
\[
     \E[U(1 - \tfrac{1}{4m}, 1 + \tfrac{1}{4m})]
\geq \frac{m}{2}
\]
and \hyperref[thm:upcrossing-inequality]{Doob's Upcrossing Inequality} gives the upper bound
\[
     \E[U(1 - \tfrac{1}{4m}, 1 + \tfrac{1}{4m})]
\leq 2m,
\]
which differs by a factor of $4$.
\ifshort\else
In \autoref{thm:tightness-Doob} we show that
\hyperref[thm:upcrossing-inequality]{Doob's Upcrossing Inequality}
can also be made exactly tight.
\fi

The lower bound for the expected number of upcrossings given in \autoref{cor:lower-bound}
is a little looser than
the upper bound given in \hyperref[thm:upcrossing-inequality]{Doob's Upcrossing Inequality}.
Closing this gap remains an open problem.
We know by \autoref{thm:upper-bound} that
given a non-summable function $f$,
the uniform probability for many $f(m)$-upcrossings goes to $0$.
However, this does not necessarily imply that expectation also tends to $0$;
low probability might be compensated for by high value.
So for expectation there might be a lower bound larger than \autoref{cor:lower-bound},
an upper bound smaller than \hyperref[thm:upcrossing-inequality]{Doob's Upcrossing Inequality},
or both.

If we drop the requirement that the rate of upcrossings be uniform,
\hyperref[thm:upcrossing-inequality]{Doob's Upcrossing Inequality}
is the best upper bound we can give\ifshort
\else:
using the little-$o$ notation,
assume there is a smaller upper bound $g(m) \in o(m)$
such that for every martingale process $(X_t)_{t \in \SetN}$,
\begin{equation}\label{eq:nonuniform-bound}
\E \left[U(1 - \tfrac{1}{m}, 1 + \tfrac{1}{m}) \right] \in o(g(m)).
\end{equation}

In the following we sketch how to construct a martingale that violates this bound.
Define $f(m) := g(m) / m$,
then $f(m) \to 0$ as $m \to \infty$,
so there is an infinite sequence $(m_i)_{i \in \SetN}$ such that
$\sum_{i=0}^\infty f(m_i) \leq 1$.
We define the martingale process $(X_t)_{t \in \SetN}$
such that it picks an $i \in \SetN$ with probability $f(m_i)$,
and then becomes a martingale
that makes \hyperref[thm:upcrossing-inequality]{Doob's Upcrossing Inequality} tight
for upcrossings between $1 - 1/m_i$ and $1 + 1/m_i$%
\ifshort.
\else:
for every $i$,
we apply the construction of \autoref{thm:tightness-Doob}.
\fi
This would give the following lower bound on the expected number of upcrossings
for each $i$:
\[
\forall i\;\;
     \E \left[ U(1 - \tfrac{1}{m_i}), 1 + \tfrac{1}{m_i}) \right]
\geq m_i f(m_i)
=    g(m_i).
\]
Since there are infinitely many $m_i$,
we get a contradiction to \eqref{eq:nonuniform-bound}.
Using a similar argument, we can show that nonuniformly,
Dubins' bound is also the best we can get.
\fi

\section{Application to the {MDL} Principle}
\label{sec:mdl-application}

Let $\mathcal{M}$ be a countable set of probability measures on $(\Sigma^\cinfty, \Foo)$,
called \emph{environment class}.
Let $K: \mathcal{M} \to [0, 1]$ be a function
such that $\sum_{Q \in \mathcal{M}} 2^{-K(Q)} \leq 1$,
called \emph{complexity function on $\mathcal{M}$}.
Following notation in \cite{Hutter:09mdltvp}, we define for $u \in \Sigma^*$ the
\emph{minimal description length} model as
\[
\MDL^u := \argmin_{Q \in \mathcal{M}} \big\{ \!-\log Q(\Gamma_u) + K(Q) \big\}.
\]
That is, $-\log Q(\Gamma_u)$ is the (arithmetic) code length of $u$ given model $Q$,
and $K(Q)$ is a complexity penalty for $Q$, also called \emph{regularizer}.
Given data $u \in \Sigma^*$,
$\MDL^u$ is the measure $Q \in \mathcal{M}$
that minimizes the total code length of data and model.

The following corollary of \autoref{thm:lower-bound} states that in some cases
the limit $\lim_{t \to \infty} \MDL^{v_{1:t}}$ does not exist with high probability.

\begin{corollary}[MDL may not converge]
\label{cor:mdl-inductively-inconsistent}
Let $P$ be a probability measure on the measurable space $(\Sigma^\cinfty, \Foo)$ with perpetual entropy.
For any $0 < \delta < 1/2$,
there is a set of probability measures $\mathcal{M}$ containing $P$,
a complexity function $K: \mathcal{M} \to [0,1]$, and
a measurable set $Z \in \Foo$ with $P(Z) \geq 1 - \delta$
such that for all $v \in Z$,
the limit $\lim_{t \to \infty} \MDL^{v_{1:t}}$ does not exist.
\end{corollary}
\begin{proof}
Fix some positive monotone decreasing summable function $f$
(e.g., the one given in \autoref{cor:lower-bound-concrete}).
Let $(X_t)_{t \in \SetN}$ be the $P$-martingale process
from \autoref{thm:lower-bound}.
By \autoref{thm:martingale-measure} there is a
probability measure $Q$ on $(\Sigma^\cinfty, \Foo)$ such that
\[
X_t(v) = \frac{Q(\Gamma_{v_{1:t}})}{P(\Gamma_{v_{1:t}})}
\]
$P$-almost surely.
Choose $\mathcal{M} := \{ P, Q \}$ with $K(P) := K(Q) := 1$.
From the definition of $\MDL$ and $Q$ it follows that
\begin{align*}
X_t(u) &< 1
\;\Longleftrightarrow\;
Q(\Gamma_u) < P(\Gamma_u)
\;\Longrightarrow\;
\MDL^u = P, \text{ and} \\
X_t(u) &> 1
\;\Longleftrightarrow\;
Q(\Gamma_u) > P(\Gamma_u)
\;\Longrightarrow\;
\MDL^u = Q.
\end{align*}
For $Z := \bigcap_{m=1}^\infty E_{m,m}^{X,f}$ \autoref{thm:lower-bound} yields
\[
     P(Z)
=    P[\forall m.\; U(1 - f(m), 1 + f(m)) \geq m]
\geq 1 - \delta.
\]
For each $v \in Z$, the measure $\MDL^{v_{1:t}}$ alternates between $P$ and $Q$
indefinitely, and thus its limit does not exist.
\end{proof}

Crucial to the proof of \autoref{cor:mdl-inductively-inconsistent} is that
not only does the process $Q/P$ oscillate indefinitely,
it oscillates around the constant $\exp(K(Q) - K(P)) = 1$.
This implies that the MDL estimator may keep changing indefinitely,
and thus it is inductively inconsistent.

\section{Bounds on Mind Changes}
\label{sec:bounds-on-mind-changes}

\newcommand{\secondfigure}{
\centering
\begin{tikzpicture}[scale=\ifshort\else 1.0 \fi]
\draw[->] (-0.1,0) -- (8.3,0) node[right] {$t$};
\ifshort
\else
\draw[->] (0,-0.1) -- (0,4) node[above] {$X_t$};
\fi

\draw[color=gray,dashed] (8.2,2) -- (-0.1,2)
	node[left,color=black] {$c$};
\draw[color=orange] (4,4) -- (4,-0.1);
\draw[color=orange] (6.5,4) -- (6.5,-0.1);

\draw[color=gray] (8.2,3) -- (-0.1,3)
	node[left,color=black] {$c + \tfrac{\alpha}{2}$};
\draw[color=gray] (8.2,1) -- (-0.1,1)
	node[left,color=black] {$c - \tfrac{\alpha}{2}$};

\draw[color=blue] (0,0.8) -- (2.5,0.8);
\draw[color=blue] (2.5,2.5) -- (3.5,2.5);
\draw[color=blue] (3.5,0.8) -- (5,0.8);
\draw[color=blue] (5,2.2) -- (5.5,2.2);
\draw[color=blue] (5.5,0.2) -- (8.1,0.2);

\draw (0, 2.8) -- (0.5, 3.1) -- (1, 2.4) -- (1.5, 2.8) -- (2, 1.5) -- (2.5, 0.5)
  -- (3, 0.7) -- (3.5, 2.8) -- (4, 3.4) -- (4.5, 1.5) -- (5, 0.2)
  -- (5.5, 2.2) -- (6, 1.8) -- (6.5, 3.1) -- (7, 3.4) -- (7.5, 3.2) -- (8, 3.5);
\draw[fill]
  (0,  2.8) circle (0.05)
  (0.5,3.1) circle (0.05)
  (1,  2.4) circle (0.05)
  (1.5,2.8) circle (0.05)
  (2,  1.5) circle (0.05)
  (2.5,0.5) circle (0.05)
  (3,  0.7) circle (0.05)
  (3.5,2.8) circle (0.05)
  (4,  3.4) circle (0.05)
  (4.5,1.5) circle (0.05)
  (5,  0.2) circle (0.05)
  (5.5,2.2) circle (0.05)
  (6,  1.8) circle (0.05)
  (6.5,3.1) circle (0.05)
  (7,  3.4) circle (0.05)
  (7.5,3.2) circle (0.05)
  (8,  3.5) circle (0.05);
\end{tikzpicture}
\caption{
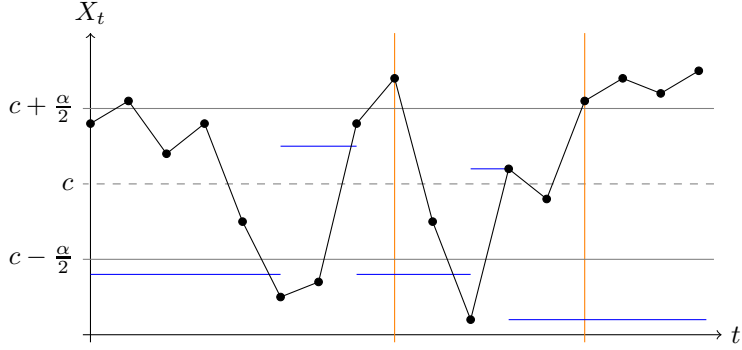
This example process has two upcrossings
between $c - \alpha/2$ and $c + \alpha/2$
(completed at the time steps of the vertical orange bars) and
four $\alpha$-alternations (completed when crossing the horizontal blue bars).
}
\label{fig:alternations-and-upcrossings}
}

Suppose we are testing a hypothesis $H \subseteq \Sigma^\cinfty$
on a stream of data $v \in \Sigma^\cinfty$.
Let $P(H \mid \Gamma_{v_{1:t}})$ denote our belief in $H$ at time $t \in \SetN$
after seeing the evidence $v_{1:t}$.
By Bayes' rule,
\[
   P(H \mid \Gamma_{v_{1:t}})
=  P(H) \frac{P(\Gamma_{v_{1:t}} \mid H)}{P(\Gamma_{v_{1:t}})}
=: X_t(v).
\]
Since $X_t$ is a constant multiple of $P(\;\cdot \mid H)/P$ and
$P(\;\cdot \mid H)$ is a probability measure on $(\Sigma^\cinfty, \Foo)$
that is absolutely continuous with respect to $P$ on cylinder sets,
the process $(X_t)_{t \in \SetN}$ is a $P$-martingale
with respect to the filtration $(\F_t)_{t \in \SetN}$
by \autoref{thm:measure-martingale}.
By definition, $(X_t)_{t \in \SetN}$ is bounded between $0$ and $1$.

Let $\alpha > 0$.
We are interested in the question
how likely it is to often change one's mind about $H$ by at least $\alpha$,
i.e., what is the probability for $X_t = P(H \mid \Gamma_{v_{1:t}})$
to decrease and subsequently increase $m$ times by at least $\alpha$.
Formally, we define the stopping times $T_{0,\nu}'(v) := 0$,
\begin{align*}
    T_{2k+1,\nu}'(v)
&:= \inf \{ t > T_{2k,\nu}'(v)
      \mid X_t(v) \leq X_{T_{2k,\nu}'(v)}(v) - \nu\alpha \}, \\
    T_{2k+2,\nu}'(v)
&:= \inf \{ t > T_{2k+1,\nu}'(v)
      \mid X_t(v) \geq X_{T_{2k+1,\nu}'(v)}(v) + \nu\alpha \},
\end{align*}
and $T_k' := \min\{ T_{k,\nu}' \mid \nu \in \{ -1, +1 \} \}$.
(In Davis' notation, $X_{T_{0,\nu}'}, X_{T_{1,\nu}'}, \ldots$
is an $\alpha$-alternating W-sequence for $\nu = 1$ and
an $\alpha$-alternating M-sequence for $\nu = -1$~\cite[Def.\ 4]{Davis:13}.)
For any $t \in \SetN$, the random variable
\[
A_t^X(\alpha)(v) := \sup \{ k \geq 0 \mid T_k'(v) \leq t \},
\]
is defined as the number of \emph{$\alpha$-alter\-nations} up to time $t$.
Let $A^X(\alpha) := \sup_{t \in \SetN} A_t^X(\alpha)$
denote the total number of $\alpha$-alternations.

Setting $\alpha = 2\eps$,
the $\alpha$-alter\-nations differ from $\eps$-upcrossings in three ways:
first, for upcrossings, the process decreases below $c - \eps$,
then increases above $c + \eps$, and then repeats.
For alternations, the process may overshoot $c - \eps$ or $c + \eps$ and
thus change the bar for the subsequent alternations,
causing a `drift' in the target bars over time.
Second, for $\alpha$-alternations the initial value of the martingale is relevant.
Third, one upcrossing corresponds to two alternations,
since one upcrossing always involves a preceding downcrossing.
See \autoref{fig:alternations-and-upcrossings}.

\ifshort\else
\begin{figure}[t]
\secondfigure
\end{figure}
\fi

To apply our bounds for upcrossings on $\alpha$-alternations,
we use the following lemma by Davis.
We reinterpret it as stating that
every bounded martingale process $(X_t)_{t \in \SetN}$
can be modified into a martingale $(Y_t)_{t \in \SetN}$
such that the probability for many $\alpha$-alternations is not decreased
and the number of alternations equals
the number of upcrossings plus the number of downcrossings%
.
\ifshort\else
A sketch of the proof can be found in Appendix \ref{app:davis-lemma}.
\fi

\begin{lemma}[{Upcrossings and alternations~\cite[Lem.\ 9]{Davis:13}}]
\label{lem:alternations-upcrossings}
Let $(X_t)_{t \in \SetN}$ be a martingale
with $0 \leq X_t \leq 1$.
There exists a martingale $(Y_t)_{t \in \SetN}$
with $0 \leq Y_t \leq 1$
and a constant $c \in (\alpha/2, 1 - \alpha/2)$
such that for all $t \in \SetN$ and for all $k \in \SetN$,
\[
     P[A_t^X(\alpha) \geq 2k]
\leq P[A_t^Y(\alpha) \geq 2k]
=    P[U_t^Y(c - \alpha/2, c + \alpha/2) \geq k].
\]
\end{lemma}

\begin{theorem}[Upper bound on alternations]
\label{thm:davis-dubins}
For every martingale process $(X_t)_{t \in \SetN}$
with $0 \leq X_t \leq 1$,
\[
     P[A(\alpha) \geq 2k]
\leq \left( \frac{1 - \alpha}{1 + \alpha} \right)^k.
\]
\end{theorem}
\begin{proof}
We apply \autoref{lem:alternations-upcrossings}
to $(X_t)_{t \in \SetN}$ and $(1 - X_t)_{t \in \SetN}$ to get
the processes $(Y_t)_{t \in \SetN}$ and $(Z_t)_{t \in \SetN}$.
\hyperref[thm:Dubins-inequality]{Dubins' Inequality} yields
\begin{align*}
      P[A_t^X(\alpha) \geq 2k]
&\leq P[U_t^Y(c_+ - \tfrac{\alpha}{2}, c_+ - \tfrac{\alpha}{2}) \geq k]
\leq  \left( \frac{c_+ - \frac{\alpha}{2}}{c_+ + \frac{\alpha}{2}} \right)^k
=:    g(c_+) \text{ and} \\
      P[A_t^{1-X}(\alpha) \geq 2k]
&\leq P[U_t^Z(c_- - \tfrac{\alpha}{2}, c_- - \tfrac{\alpha}{2}) \geq k]
\leq  \left( \frac{c_- - \frac{\alpha}{2}}{c_- + \frac{\alpha}{2}} \right)^k
=     g(c_-)
\end{align*}
for some $c_+, c_- \in (\alpha/2, 1 - \alpha/2)$.
Because \autoref{lem:alternations-upcrossings} is symmetric for
$(X_t)_{t \in \SetN}$ and $(1 - X_t)_{t \in \SetN}$,
we have $c_+ = 1 - c_-$.
Since $P[A_t^X(\alpha) \geq 2k] = P[A_t^{1 - X}(\alpha) \geq 2k]$
by the definition of $A_t^X(\alpha)$,
we have that both are less than $\min\{g(c_+), g(c_-)\} = \min\{g(c_+), g(1 - c_+)\}$.
This is maximized for $c_+ = c_- = 1/2$
because $g$ is strictly monotone increasing for $c > \alpha / 2$.
Therefore
\[
     P[A_t^X(\alpha) \geq 2k]
\leq \left( \frac{\frac{1}{2} - \frac{\alpha}{2}}{\frac{1}{2} + \frac{\alpha}{2}} \right)^k
=    \left( \frac{1 - \alpha}{1 + \alpha} \right)^k.
\]
Since this bound is independent of $t$,
it also holds for $P[A^X(\alpha) \geq 2k]$.
\end{proof}

The bound of \autoref{thm:davis-dubins} is the square root of the bound
derived by Davis~\cite[Thm.\ 10 \& Thm.\ 11]{Davis:13}.
\begin{equation}\label{eq:Davis-bound}
     P[A(\alpha) \geq 2k]
\leq \left( \frac{1 - \alpha}{1 + \alpha} \right)^{2k}
\end{equation}
This bound is tight~\cite[Cor.\ 13]{Davis:13}.
A similar bound for upcrossings was proved by Dubins~\cite[Cor.\ 1]{Dubins:72}.

Because $0 \leq X_t \leq 1$,
the process $(1 - X_t)_{t \in \SetN}$ is also a nonnegative martingale,
hence the same upper bounds apply to it.
This explains why the result in \autoref{thm:davis-dubins} is worse than Davis' bound~\eqref{eq:Davis-bound}:
Dubins' bound applies to all nonnegative martingales,
while Davis' bound uses the fact that the process is bounded from below \emph{and} above.
For unbounded nonnegative martingales,
downcrossings are `free' in the sense that
one can make a downcrossing almost surely successful
(as done in the proof of \autoref{thm:lower-bound}).
If we apply Dubins' bound to the process $(1 - X_t)_{t \in \SetN}$,
we get the same probability bound for the downcrossings of $(X_t)_{t \in \SetN}$
(which are upcrossings of $(1 - X_t)_{t \in \SetN}$).
Multiplying both bounds yields Davis' bound~\eqref{eq:Davis-bound};
however, we still require a formal argument
why the upcrossing and downcrossing bounds are independent.

The following corollary to \autoref{thm:davis-dubins}
derives an upper bound on the \emph{expected} number of $\alpha$-alternations.

\begin{theorem}[Upper bound on expected alternations]
\label{thm:expected-alternations}
For every martingale $(X_t)_{t \in \SetN}$
with $0 \leq X_t \leq 1$,
the expectation
$
     \E[A(\alpha)]
\leq \tfrac{1}{\alpha}.
$
\end{theorem}
\begin{proof}
By \autoref{thm:davis-dubins} we have
$P[A(\alpha) \geq 2k] \leq \left( \frac{1 - \alpha}{1 + \alpha} \right)^k$,
and thus
\begin{align*}
      \E[A(\alpha)]
&=    \sum_{k=1}^\infty P[A(\alpha) \geq k] \\
&=    P[A(\alpha) \geq 1] + \sum_{k=1}^\infty \big(
        P[A(\alpha) \geq 2k] + P[A(\alpha) \geq 2k + 1]
      \big) \\
&\leq 1 + \sum_{k = 1}^\infty 2 P[A(\alpha) \geq 2k]
 \leq 1 + 2\sum_{k = 1}^\infty \left( \frac{1 - \alpha}{1 + \alpha} \right)^k
= \frac{1}{\alpha}.
\qedhere
\end{align*}
\end{proof}

We now apply the technical results of this section
to the martingale process $X_t = P(\;\cdot \mid H) / P$,
our belief in the hypothesis $H$ as we observe data.
The probability of changing our mind $k$ times by at least $\alpha$
decreases exponentially with $k$ (\autoref{thm:davis-dubins}).
Furthermore, the expected number of times we change our mind by at least $\alpha$
is bounded by $1/\alpha$ (\autoref{thm:expected-alternations}).
In other words, having to change one's mind a lot often is unlikely.

Because in this section we consider martingales that are bounded between $0$ and $1$,
the lower bounds from \autoref{sec:indefinitely-oscillating-martingales} do not apply here.
While for the martingales constructed in \autoref{thm:lower-bound},
the number of $2\alpha$-alternations and the number of $\alpha$-up- and downcrossings coincide,
these processes are not bounded.
However, we can give a similar construction
that is bounded between $0$ and $1$ and makes Davis' bound asymptotically tight.

\section{Conclusion}
\label{sec:conclusion}

We constructed an indefinitely oscillating martingale process
from a summable function $f$.
\autoref{thm:lower-bound} and \autoref{cor:lower-bound}
give uniform lower bounds on the probability and expectation of
the number of upcrossings of decreasing magnitude.
In \autoref{thm:upper-bound} we proved the corresponding upper bound
if the function $f$ is not summable.
In comparison,
\hyperref[thm:upcrossing-inequality]{Doob's Upcrossing Inequality}
and \hyperref[thm:Dubins-inequality]{Dubins' Inequality}
give upper bounds that are not uniform.
In \autoref{sec:upper-bounds} we showed that
for a certain summable function $f$,
our martingales make these bounds asymptotically tight as well.

Our investigation of indefinitely oscillating martingales
was motivated by two applications.
First, in \autoref{cor:mdl-inductively-inconsistent} we showed that
the minimum description length operator may not exist in the limit:
for any probability measure $P$
we can construct a probability measure $Q$ such that $Q/P$ oscillates forever
around the specific constant that causes $\lim_{t \to \infty} \MDL^{v_{1:t}}$ to not converge.

Second, we derived bounds for the probability of changing one's mind about a hypothesis $H$
when observing a stream of data $v \in \Sigma^\cinfty$.
The probability $P(H \mid \Gamma_{v_{1:t}})$ is a martingale and
in \autoref{thm:davis-dubins} we proved that
the probability of changing the belief in $H$ often by at least $\alpha$ decreases exponentially.

A question that remains open is
whether there is a \emph{uniform} upper bound on the \emph{expected} number of upcrossings
tighter than \hyperref[thm:upcrossing-inequality]{Doob's Upcrossing Inequality}.


\bibliographystyle{alpha}
\bibliography{references}

\ifshort\else
\newpage\appendix
\section{Appendix}
\label{sec:appendix}

\addtocounter{dummy}{2} 

\subsection{Notation}
\label{app:notation}

\begin{itemize}
\item $:=$ denotes a definition.
\item $A^c := \Sigma^\cinfty \setminus A$ denotes the complement of a measurable set $A \subseteq \Sigma^\cinfty$.
\item $\uplus$ denotes disjoint union.
\item For a set $X$, the power set of $X$ is denoted by $2^X$.
\item $\mathbbm{1}_X$ is the characteristic function for a set $X$, i.e.,
	$\mathbbm{1}_X(x) = 1$ if $x \in X$ and $0$ otherwise.
\item $\omega$ is the smallest infinite ordinal.
\item $\SetN$ is the set of natural numbers.
\item $\mathbb{R}$ is the set of real numbers.
\item For $a, b \in \mathbb{R}$,
	$[a, b]$ denotes the closed interval with end points $a$ and $b$;
	$(a, b]$ and $[a, b)$ denote half-open intervals and
	$(a, b)$ denotes an open interval.
\item The set $\Sigma$ denotes a finite alphabet.
	The set of all finite strings of length $n$ is denoted $\Sigma^n$,
	the set of all finite strings is denoted $\Sigma^*$, and
	the set of all infinite strings is denoted $\Sigma^\omega$.
\item For a string $u \in \Sigma^*$, $|u|$ denotes the length of $u$.
\item For $v \in \Sigma^\cinfty$, $v_{1:k}$ denotes the first $k$ characters of $v$.
\item $f \in \Omega(g)$ denotes $g \in O(f)$, i.e.,
	$\exists k > 0\; \exists x_0\; \forall x \geq x_0.\; g(x) \cdot k \leq f(x)$.
\item $f \in o(g)$ denotes $\lim_{x \to \infty} \frac{f(x)}{g(x)} = 0$.
\end{itemize}

\subsection{Measures and Martingales}
\label{app:measures-martingales}

In this section we prove
\autoref{thm:measure-martingale} and \autoref{thm:martingale-measure},
establishing the connecting between measures on infinite strings and martingales.

\begin{proof}[Proof of \autoref{thm:measure-martingale}]
$X_t$ is only undefined if $P(\Gamma_{v_{1:t}}) = 0$.
The set
\[
\{ v \in \Sigma^\cinfty \mid \exists t.\; P(\Gamma_{v_{1:t}}) = 0 \}
\]
has $P$-measure $0$ and hence
$(X_t)_{t \in \SetN}$ is well-defined almost everywhere.

$X_t$ is constant on $\Gamma_u$ for all $u \in \Sigma^t$, and
$\F_t$ is generated by a collection of finitely many disjoint sets:
\[
\Sigma^\cinfty = \biguplus_{u \in \Sigma^t} \Gamma_u.
\]
\begin{enumerate}[(a)]
\item
Therefore $X_t$ is $\F_t$-measurable.

\item
$\Gamma_u = \biguplus_{a \in \Sigma} \Gamma_{ua}$ for all $u \in \Sigma^t$ and $v \in \Gamma_u$,
and therefore
\begin{align*}
\E[X_{t+1} \mid \F_t](v)
&= \frac{1}{P(\Gamma_u)} \sum_{a \in \Sigma} X_{t+1}(ua) P(\Gamma_{ua})
 = \frac{1}{P(\Gamma_u)} \sum_{a \in \Sigma} \frac{Q(\Gamma_{ua})}{P(\Gamma_{ua})} P(\Gamma_{ua}) \\
&\stackrel{(\ast)}{=} \frac{1}{P(\Gamma_u)} \sum_{a \in \Sigma} Q(\Gamma_{ua})
 = \frac{Q(\Gamma_u)}{P(\Gamma_u)}
 = X_t(v).
\end{align*}
At $(\ast)$ we used the fact that
$Q$ is absolutely continuous with respect to $P$ on cylinder sets.
(If $Q$ were not absolutely continuous with respect to $P$ on cylinder sets
there are cases where
$P(\Gamma_u) > 0$, $P(\Gamma_{ua}) = 0$, and $Q(\Gamma_{ua}) \neq 0$.
Therefore $X_{t+1}(ua)$ does not contribute to the expectation and thus
$X_{t+1}(ua) P(\Gamma_{ua}) = 0 \neq Q(\Gamma_{ua})$.)
\end{enumerate}
$P \geq 0$ and $Q \geq 0$ by definition, thus $X_t \geq 0$.
Since $P(\Gamma_\epstr) = Q(\Gamma_\epstr) = 1$,
we have $\E[X_0] = 1$.
\end{proof}

The following lemma gives a convenient condition for
the existence of a measure on $(\Sigma^\omega, \Foo)$.
It is a special case of
the Daniell-Kolmogorov Extension Theorem~\cite[Thm.\ 26.1]{RW:1994}.

\begin{lemma}[Extending measures]
\label{lem:semimeasure}
Let $q: \Sigma^* \to [0, 1]$ be a function such that
$q(\epstr) = 1$ and
$\sum_{a \in \Sigma} q(ua) = q(u)$ for all $u \in \Sigma^*$.
Then there exists a unique probability measure $Q$
on $(\Sigma^\cinfty, \Foo)$ such that
$q(u) = Q(\Gamma_u)$ for all $u \in \Sigma^*$.
\end{lemma}

To prove this lemma, we need the following two ingredients.

\begin{definition}[Semiring]
\label{def:semiring}
A set $\mathcal{R} \subseteq 2^\Omega$ is called \emph{semiring over $\Omega$} iff
\begin{enumerate}[(a)]
\item $\emptyset \in \mathcal{R}$,
\item for all $A, B \in \mathcal{R}$, the set $A \cap B \in \mathcal{R}$, and
\item for all $A, B \in \mathcal{R}$,
	there are pairwise disjoint sets $C_1, \ldots, C_n \in \mathcal{R}$
	such that $A \setminus B = \biguplus_{i=1}^n C_i$.
\end{enumerate}
\end{definition}

\begin{theorem}[{Carathéodory's Extension Theorem~\cite[Thm.\ A.1.1]{Durrett:10}}]
\label{thm:caratheodory}
Let $\mathcal{R}$ be a semiring over $\Omega$ and
let $\mu: \mathcal{R} \to [0,1]$ be a function such that
\begin{enumerate}[(a)]
\item $\mu(\Omega) = 1$
	\emph{(normalization)},
\item $\mu(\biguplus_{i=1}^n A_i) = \sum_{i=1}^n \mu(A_i)$
	for pairwise disjoint sets $A_1, \ldots, A_n \in \mathcal{R}$
	such that $\biguplus_{i=1}^n A_i \in \mathcal{R}$
	\emph{(finite additivity)}, and
\item $\mu(\bigcup_{i \geq 0} A_i) \leq \sum_{i \geq 0} \mu(A_i)$
	for any collection $(A_i)_{i \geq 0}$ such that
	each $A_i \in \mathcal{R}$ and $\bigcup_{i \geq 0} A_i \in \mathcal{R}$
	\emph{($\sigma$-subadditivity)}.
\end{enumerate}
Then there is a unique extension $\overline{\mu}$ of $\mu$
that is a probability measure on $(\Omega, \sigma(\mathcal{R}))$ such that
$\overline{\mu}(A) = \mu(A)$ for all $A \in \mathcal{R}$.
\end{theorem}

\begin{proof}[Proof of \autoref{lem:semimeasure}]
We show the existence of $Q$ using
\hyperref[thm:caratheodory]{Carathéodory's Extension Theorem}.
Define $\mathcal{R} := \{ \Gamma_u \mid u \in \Sigma^* \} \cup \{ \emptyset \}$.
\begin{enumerate}[(a)]
\item $\emptyset \in \mathcal{R}$.

\item For any $\Gamma_u, \Gamma_v \in \mathcal{R}$, either
\begin{itemize}
\item $u$ is a prefix of $v$ and $\Gamma_u \cap \Gamma_v = \Gamma_v \in \mathcal{R}$, or
\item $v$ is a prefix of $u$ and $\Gamma_u \cap \Gamma_v = \Gamma_u \in \mathcal{R}$, or
\item $\Gamma_u \cap \Gamma_v = \emptyset \in \mathcal{R}$.
\end{itemize}

\item For any $\Gamma_u, \Gamma_v \in \mathcal{R}$,
\begin{itemize}
\item $\Gamma_u \setminus \Gamma_v = \biguplus_{w \in \Sigma^{|v| - |u|} \setminus \{ x \}} \Gamma_{uw}$
	if $v = ux$, i.e., $u$ is a prefix of $v$, and
\item $\Gamma_u \setminus \Gamma_v = \emptyset$ otherwise.
\end{itemize}
\end{enumerate}
Therefore $\mathcal{R}$ is a semiring.
By definition of $\mathcal{R}$, we have $\sigma(\mathcal{R}) = \Foo$.

The function $q: \Sigma^* \to [0,1]$ naturally gives rise to a function
$\mu: \mathcal{R} \to [0,1]$ with $\mu(\emptyset) := 0$ and
$\mu(\Gamma_u) := q(u)$ for all $u \in \Sigma^*$.
We will now check the prerequisites of
\hyperref[thm:caratheodory]{Carathéodory's Extension Theorem}.
\begin{enumerate}[(a)]
\item (Normalization.)
$\mu(\Sigma^\cinfty) = \mu(\Gamma_\epstr) = q(\epstr) = 1$.

\item (Finite additivity.)
Let $\Gamma_{u_1}, \ldots, \Gamma_{u_k} \in \mathcal{R}$ be pairwise disjoint sets such that
$\Gamma_w := \biguplus_{i=1}^k \Gamma_{u_i} \in \mathcal{R}$.
Let $\ell := \max \{ |u_i| \mid 1 \leq i \leq k \}$, then
$\Gamma_w = \biguplus_{v \in \Sigma^\ell} \Gamma_{wv}$.
By assumption, $\sum_{a \in \Sigma} q(ua) = q(u)$,
thus $\sum_{a \in \Sigma} \mu(\Gamma_{ua}) = \mu(\Gamma_u)$ and inductively
we have
\begin{equation}\label{eq-semimeasures-1}
\mu(\Gamma_{u_i}) = \sum_{s \in \Sigma^{\ell - |u_i|}} \mu(\Gamma_{u_i s}),
\end{equation}
and
\begin{equation}\label{eq-semimeasures-2}
\mu(\Gamma_w) = \sum_{v \in \Sigma^\ell} \mu(\Gamma_{wv}).
\end{equation}
For every string $v \in \Sigma^\ell$,
the concatenation $wv \in \Gamma_w = \biguplus_{i=1}^k \Gamma_{u_i}$,
so there is a unique $i$ such that $wv \in \Gamma_{u_i}$.
Hence there is a unique string $s \in \Sigma^{\ell - |u_i|}$ such that $wv = u_i s$.
Together with \eqref{eq-semimeasures-1} and \eqref{eq-semimeasures-2} this yields
\[
  \mu \left( \biguplus_{i=1}^k \Gamma_{u_i} \right)
= \mu(\Gamma_w)
= \sum_{v \in \Sigma^\ell} \mu(\Gamma_{wv})
= \sum_{i=1}^k \sum_{s \in \Sigma^{\ell - |u_i|}} \mu(\Gamma_{u_i s})
= \sum_{i=1}^k \mu(\Gamma_{u_i}).
\]

\item ($\sigma$-subadditivity.)
We will show that each $\Gamma_u$ is compact
with respect to the topology $\mathcal{O}$ generated by $\mathcal{R}$.
$\sigma$-subadditivity then follows from (b)
because every countable union is in fact a finite union.

We will show that the topology $\mathcal{O}$ is the product topology
of the discrete topology on $\Sigma$.
(This establishes that $(\Sigma^\omega, \mathcal{O})$ is a Cantor Space.)
Every projection $\pi_k: \Sigma^\cinfty \to \Sigma$ selecting the $k$-th symbol is continuous,
since $\pi_k^{-1}(a) = \bigcup_{u \in \Sigma^{k-1}} \Gamma_{ua}$ for every $a \in \Sigma$.
Moreover, $\mathcal{O}$ is the coarsest topology with this property, since we can generate
every open set $\Gamma_u \in \mathcal{R}$ in the base of the topology by
\[
\Gamma_u = \bigcap_{i=1}^{|u|} \pi_i^{-1}(\{ u_i \}).
\]

The set $\Sigma$ is finite and thus compact.
By Tychonoff's Theorem, $\Sigma^\cinfty$ is also compact.
Therefore $\Gamma_u$ is compact since it is homeomorphic to $\Sigma^\cinfty$
via the canonical map $\beta_u: \Sigma^\cinfty \to \Gamma_u$, $v \mapsto uv$.
\end{enumerate}
From (a), (b), and (c) \hyperref[thm:caratheodory]{Carathéodory's Extension Theorem} yields
a unique probability measure $Q$ on $(\Sigma^\cinfty, \Foo)$ such that
$Q(\Gamma_u) = \mu(\Gamma_u) = q(u)$ for all $u \in \Sigma^*$.
\end{proof}

Using \autoref{lem:semimeasure},
the proof of \autoref{thm:martingale-measure} is now straightforward.

\begin{proof}[Proof of \autoref{thm:martingale-measure}]
We define a function $q: \Sigma^* \to \mathbb{R}$, with
\[
q(u) := X_{|u|}(v) P(\Gamma_u)
\]
for any $v \in \Gamma_u$.
The choice of $v$ is irrelevant because $X_{|u|}$ is constant on $\Gamma_u$
since it is $\F_t$-measurable.
In the following,
we also write $X_t(u)$ if $|u| = t$ to simplify notation.

The function $q$ is non-negative because $X_t$ and $P$ are both non-negative.
Moreover, for any $u \in \Sigma^t$,
\[
     1
=    \E[X_t]
=    \int_{\Sigma^\cinfty} X_t dP
\geq \int_{\Gamma_u} X_t dP
=    P(\Gamma_u) X_t(u)
=    q(u).
\]
Hence the range of $q$ is a subset of $[0, 1]$.

We have $q(\epstr) = X_0(\epstr) P(\Gamma_\epstr) = \E[X_0] = 1$
since $P$ is a probability measure and
$\F_0 = \{ \emptyset, \Sigma^\cinfty \}$ is the trivial $\sigma$-algebra.
Let $u \in \Sigma^t$.
\begin{align*}
   \sum_{a \in \Sigma} q(ua)
&= \sum_{a \in \Sigma} X_{t+1}(ua) P(\Gamma_{ua})
 = \int_{\Gamma_u} X_{t+1} dP \\
&= \int_{\Gamma_u} \E[X_{t+1} \mid \F_t] dP
 = \int_{\Gamma_u} X_t dP
 = P(\Gamma_u) X_t(u)
 = q(u).
\end{align*}
By \autoref{lem:semimeasure},
there is a probability measure $Q$ on $(\Sigma^\cinfty, \Foo)$
such that $q(u) = Q(\Gamma_u)$ of all $u \in \Sigma^*$.
Therefore, for all $v \in \Sigma^\cinfty$ and
for all $t \in \SetN$ with $P(\Gamma_{v_{1:t}}) > 0$,
\[
  X_t(v)
= \frac{q(v_{1:t})}{P(\Gamma_{v_{1:t}})}
= \frac{Q(\Gamma_{v_{1:t}})}{P(\Gamma_{v_{1:t}})}.
\]
Moreover,
$Q$ is absolutely continuous with respect to $P$ on cylinder sets since
$P(\Gamma_u) = 0$ implies
\[
Q(\Gamma_u) = q(u) = X_{|u|}(u) P(\Gamma_u) = 0.
\qedhere
\]
\end{proof}

\subsection{Different Upcrossing inequalities and their tightness}
\label{app:tightness}

There are different versions of the upcrossing inequality in circulation.
Let $a < b$ and let $(X_t)_{t \in \SetN}$ be a martingale process.
Doob~\cite[VII§3 Thm.\ 3.3]{Doob:53} states
\begin{equation}\label{eq:upcrossing-inequality-Doob}
     \E[U_t(a, b)]
\leq \tfrac{1}{b - a} \E[\max\{ X_t - a, 0 \}].
\end{equation}
Durrett~\cite[Thm.\ 5.2.7]{Durrett:10} gives a slightly stronger version:
\begin{equation}\label{eq:upcrossing-inequality-Durrett}
     \E[U_t(a, b)]
\leq \tfrac{1}{b - a} \Big(
       \E[\max\{ X_t - a, 0 \}] - \E[\max\{ X_0 - a, 0 \}]
     \Big).
\end{equation}
We will prove tight
the version stated in \autoref{thm:upcrossing-inequality}~\cite[Thm.\ 1.1]{Xu12}:
\begin{equation}\label{eq:upcrossing-inequality-Xu}
     \E[U_t(a, b)]
\leq \tfrac{1}{b - a} \E[\max\{ a - X_t, 0 \}].
\end{equation}
For nonnegative martingales
we can estimate $\E[\max\{ a - X_t, 0 \}] \leq a$
to get a bound independent of $t$ from
the upcrossing inequality \eqref{eq:upcrossing-inequality-Xu}.
To get a bound independent of $t$ from
\eqref{eq:upcrossing-inequality-Doob} or
\eqref{eq:upcrossing-inequality-Durrett},
we look at the upcrossings of the martingale process $(-X_t)_{t \in \SetN}$,
which are the downcrossings of $(X_t)_{t \in \SetN}$.
The number of downcrossings differs from the number of upcrossings by at most $1$,
so we can conclude from \eqref{eq:upcrossing-inequality-Durrett},
\begin{align*}
      \E[U_t^X(a, b)]
&\leq \E[U_t^{-X}(-b, -a)] + 1 \\
&\leq \tfrac{1}{b - a} \Big(
        \E[\max\{ a - X_t, 0 \}] - \E[\max\{ a - X_0, 0 \}]
      \Big) + 1.
\end{align*}

The origin of the diversity in upcrossing inequalities stems from
the details of their proofs.
When we start betting
every time the process $(X_t)_{t \in \SetN}$ falls below $a$ and stop
every time it rises above $b$,
our gain at time $t$ is at least $(b - a) U_t(a, b)$
plus some amount $R$ that we gained or lost since we started betting last time
in case the last upcrossing has not yet completed.
Because we are betting on a martingale,
our expected gain is zero, hence $(b - a) \E[U_t(a, b)] = \E[-R]$.
The right hand sides of the equations \eqref{eq:upcrossing-inequality-Doob},
\eqref{eq:upcrossing-inequality-Durrett}, and
\eqref{eq:upcrossing-inequality-Xu}
arise from the way we estimate $R$ from below.
The inequality \eqref{eq:upcrossing-inequality-Xu} estimates $R$ by taking
into account any possible losses ignoring gains
since we last started betting at $a$.
Contrarily, \eqref{eq:upcrossing-inequality-Doob} estimates $R$ by taking
into account any possible gains ignoring losses
since we started betting at $a$.
In \eqref{eq:upcrossing-inequality-Durrett} we additionally suppose that we
are betting starting at time $0$ and take into account any losses before
$X_t$ falls below $a$ for the first time.

\begin{lemma}[Tightness Criterion for \eqref{eq:upcrossing-inequality-Xu}]
\label{lem:tightness-criterion}
Let $a < b$ and
let $(X_t)_{t \in \SetN}$ be a martingale such that
\begin{enumerate}[(a)]
\item $X_t$ does not assume any values between $a$ and $b$, and
\item all upcrossings are completed at $b$ and
all downcrossings are completed at $a$:
\[
X_{T_{2k}} = b
\quad\text{and}\quad
X_{T_{2k+1}} = a
\quad \forall k \in \SetN.
\]
\end{enumerate}
Then the inequality \eqref{eq:upcrossing-inequality-Xu} is tight, i.e.,
\[
  \E[U_t(a, b)]
= \tfrac{1}{b - a} \E[\max\{ a - X_t, 0 \}].
\]
\end{lemma}
\begin{proof}
This proof essentially follows the proof of Doob's Upcrossing Inequality
given in \cite{Xu12}.
Define the process
\[
   D_t(v)
:= \sum_{k=1}^\infty \big(
     X_{\min\{t, T_{2k}\}}(v) - X_{\min\{t, T_{2k-1}\}}(v)
   \big).
\]
Since all but finitely many terms in the infinite sum are zero,
$D_t$ is well-defined.

The process $(D_t)_{t \in \SetN}$ is martingale:
\[
  \E[D_{t+1} \mid \F_t]
= \sum_{k=1}^\infty \big(
    \E[ X_{\min\{t+1, T_{2k}\}} \mid \F_t]
    - \E[ X_{\min\{t+1, T_{2k-1}\}} \mid \F_t]
  \big).
\]
Fix some $i \in \SetN$.
Conditioning on $\F_t$, we know whether $T_i > t$ or $T_i \leq t$
since $T_i$ is a stopping time.
In case $T_i > t$ we have
$t + 1 \leq T_i$, implying
$X_{\min\{t+1, T_i\}} = X_{t+1}$ and thus
$\E[ X_{\min\{t+1, T_i\}} \mid \F_t] = \E[ X_{t+1} \mid \F_t] = X_t = X_{\min\{t, T_i\}}$
because $(X_t)_{t \in \SetN}$ is martingale.
In case $T_i \leq t$ we have
$X_{\min\{t+1, T_i\}} = X_{T_i}$ and hence
$\E[ X_{\min\{t+1, T_i\}} \mid \F_t] = \E[ X_{T_i} \mid \F_t] = X_{T_i} = X_{\min\{t, T_i\}}$.
In both cases we get
$\E[X_{\min\{t+1, T_i\}} \mid \F_t] = X_{\min\{t, T_i\}}$,
therefore $\E[D_{t+1} \mid \F_t] = D_t$.

Let $t \in \SetN$ be some time step, and fix $v \in \Sigma^\omega$.
Let $U_t := U_t(a, b)$ denote
the number of upcrossings that have been completed up to time $t$.
We distinguish the following two cases.
\begin{enumerate}[(i)]
\item There is an incomplete upcrossing,
	$T_{2U_t+1} \leq t < T_{2U_t+2}$.
\item There is no incomplete upcrossing,
	$T_{2U_t} \leq t < T_{2U_t+1}$.
\end{enumerate}

In case (i) we have $X_t < b$ and therefore $X_t \leq a$ by assumption (a).
With assumption (b) we get
\begin{equation}\label{eq:D1}
\begin{aligned}
   D_t
&= \sum_{k=1}^{U_t} (X_{T_{2k}} - X_{T_{2k-1}}) + X_t - X_{T_{2U_t+1}} \\
&= \sum_{k=1}^{U_t} (b - a) + X_t - a
 = (b - a)U_t + X_t - a.
\end{aligned}
\end{equation}

In case (ii) we have $X_t > a$.
With assumption (b) we get
\begin{equation}\label{eq:D2}
  D_t
= \sum_{k=1}^{U_t} (X_{T_{2k}} - X_{T_{2k-1}})
= \sum_{k=1}^{U_t} (b - a)
= (b - a)U_t.
\end{equation}
From \eqref{eq:D1} and \eqref{eq:D2} follows that
\[
D_t = (b - a) U_t + \min\{ X_t - a, 0 \}.
\]
We have $D_0 = 0$ and since $(D_t)_{t \in \SetN}$ is martingale
it follows that $\mathbb{E}[D_t] = 0$.
Hence
\[
  (b - a) \mathbb{E}[U_t]
= \mathbb{E}[-\min\{ X_t - a, 0 \}]
= \mathbb{E}[\max\{ a - X_t, 0 \}].
\qedhere
\]
\end{proof}

\begin{theorem}[Tightness of Doob's Upcrossing Inequality]
\label{thm:tightness-Doob}
Let $P$ be a probability measure with perpetual entropy.
For all $b > a > 0$
there is a nonnegative martingale $(X_t)_{t \in \SetN}$ with $X_0 = a$
that makes \hyperref[thm:upcrossing-inequality]{Doob's Upcrossing Inequality} tight
for all $t > 0$:
\[
  0
< \E[U_t^X(a, b)]
= \tfrac{1}{b - a} \E[\max\{ a - X_t, 0 \}].
\]
\end{theorem}
We added the requirement $\E[U_t^X(a, b)] > 0$,
because otherwise the constant process $X_t = a$
would trivially make the inequality tight.

\begin{proof}
Fix $c := (a + b) / 2$ and
set $f(t) := (b - a)/(b + a) = (b - a) / (2c)$;
then $f(t) < 1$ because $a > 0$.
Define $X_0 := X_1 := a / c$.
The function $f$ is not summable,
but we nonetheless apply the same construction as in \autoref{thm:lower-bound}:
for $t > 1$ let $X_t$ be defined as in the proof of \autoref{thm:lower-bound}.
We prove that
the scaled process $Y_t := c \cdot X_t$
makes the inequality \eqref{eq:upcrossing-inequality-Xu} tight.
Since $(X_t)_{t \in \SetN}$ does upcrossings between
$1 - f(M_t) = a / c$ and $1 + f(M_t) = b / c$,
the scaled process $(Y_t)_{t \in \SetN}$ does upcrossings between $a$ and $b$.

By \autoref{claim:is-martingale} $(X_t)_{t \in \SetN}$ is a martingale process,
and by \autoref{claim:nonnegative-and-expectation} $X_t \geq 0$,
hence this also applies to the scaled process $(Y_t)_{t \in \SetN}$.
We check the criterion given in \autoref{lem:tightness-criterion}.
\begin{enumerate}[(a)]
\item
	This holds for $X_t$ for $t = 0$ and $t = 1$
	according to the definition of $(X_t)_{t \in \SetN}$.
	For $t > 1$ this follows from \autoref{claim:X_t-jumps}
	since $T_1 = 1$ because $X_1 = a / c$.

\item From \autoref{claim:bounds-on-X_t+1},
	this is fulfilled in cases (i) and (ii).
	In case (iii) we have $X_{t+1} \leq 1 - f(M_t)$,
	so the process cannot do an (up-)crossing.
\end{enumerate}
It remains to show that $\E[U_t^Y(a, b)] > 0$.
By \autoref{claim:gamma-is-small} $X_0 > \gamma_0$,
so for all $v \in \Sigma^\omega$ with $v_0 = a_\epstr$
we have $X_1 = 1 + f(M_t)$,
therefore $U^X_1(a / c, b / c)(v) \geq 1$.
Since $P(\Gamma_{a_\epstr}) > 0$ by assumption,
this yields $\E[U^X_1(a / c, b / c)] > 0$
and hence $\E[U^Y_t(a, b)] > 0$ for all $t > 1$.
\end{proof}

The process from \autoref{thm:tightness-Doob} also gives a tightness result
as $t \to \infty$.
A weaker lower bound
$\E[U^X(a, b)] \geq \tfrac{a + b}{8(b - a)} - \tfrac{1}{2}$
can be derived directly from \autoref{cor:lower-bound}
using $\delta := 1/2$ and
\[
f(i) :=
\begin{cases}
\tfrac{b - a}{b + a}, &\text{if } i \leq \tfrac{b + a}{4(b - a)}, \\
0,                    &\text{otherwise},
\end{cases}
\]
and scaling the process with $(b+a)/2$.

\begin{corollary}[Asymptotic tightness of Doob's Upcrossing Inequality]
\label{cor:asymptotic-tightness-Doob}
Let $P$ be a probability measure with perpetual entropy.
For all $b > a > 0$
there is a nonnegative martingale $(X_t)_{t \in \SetN}$ with $X_0 = a$
such that
\[
  \E[U^X(a, b)]
= \tfrac{a}{b - a}.
\]
\end{corollary}
\begin{proof}
Consider the process $(Y_t)_{t \in \SetN}$
from the proof of \autoref{thm:tightness-Doob}.
Since $(X_t)_{t \in \SetN}$ is a nonnegative martingale,
the Martingale Convergence Theorem~\cite[Thm.\ 5.2.8]{Durrett:10} implies that
$(X_t)_{t \in \SetN}$ converges almost surely to a limit $X_\omega \geq 0$.
This limit can only be $0$ or $1$ by \autoref{claim:X-does-not-converge}.
Since $f(t) = (b - a)/(2c) > 0$ for all $t$,
$(X_t)_{t \in \SetN}$ does not converge to $1$ by \autoref{claim:X_t-jumps}
($T_1 = 1$ by construction).
Thus $X_\omega = 0$ almost surely,
but this generally does \emph{not} imply
$\lim_{t \to \infty} \E[X_t] = 0$~\cite[Ex.\ 5.2.3]{Durrett:10}.
However, $\max\{a - Y_t, 0\} = \max\{a - c X_t, 0\}$
is bounded, therefore uniformly integrable.
By \cite[Thm.\ 5.5.2]{Durrett:10}
(a generalization of the dominated convergence theorem),
\begin{equation}\label{eq:limit-is-a}
  \lim_{t \to \infty} \E[\max\{a - Y_t, 0\}]
= \E[\max\{a - c X_\omega, 0 \}]
= a.
\end{equation}
By \hyperref[thm:Dubins-inequality]{Dubins' Inequality},
\begin{align*}
      \E[U_t^Y(a,b) \cdot \mathbbm{1}_{U_t^Y(a,b) \geq k}]
&=    \sum_{i=k}^\infty P[U_t^Y(a,b) \geq i]
 \leq \sum_{i=k}^\infty a^i b^{-i} \\
&=    a^k b^{-k} \tfrac{b}{b - a}
\to  0 \text{ as } k \to \infty,
\end{align*}
hence $U_t^Y(a,b)$ is also uniformly integrable and
by the same theorem~\cite[Thm.\ 5.5.2]{Durrett:10} and \eqref{eq:limit-is-a},
\[
  (b - a) \E[U^Y(a,b)]
= \lim_{t \to \infty} (b - a) \E[U_t^Y(a,b)]
= \lim_{t \to \infty} \E[\max\{ a - Y_t, 0 \}]
= a.
\qedhere
\]
\end{proof}

The same process can also be used to show that
\hyperref[thm:Dubins-inequality]{Dubins' Inequality} is tight.
For a specific underlying probability measure
a proof of this is sketched by Dubins~\cite[Thm.\ 12.1]{Dubins:62}.
We prove a version that is agnostic with respect to the probability measure $P$.

\begin{corollary}[Tightness of Dubins' Inequality]
\label{cor:tightness-Dubins}
Let $P$ be a probability measure with perpetual entropy.
For all $b > a > 0$
there is a nonnegative martingale $(X_t)_{t \in \SetN}$ with $X_0 = a$
that makes \hyperref[thm:Dubins-inequality]{Dubins' Inequality} tight:
\[
  P[U^X(a, b) \geq k]
= \frac{a^k}{b^k}
\]
\end{corollary}
\begin{proof}
We use \hyperref[thm:Dubins-inequality]{Dubins' Inequality}
on the process $(X_t)_{t \in \SetN}$
from \autoref{cor:asymptotic-tightness-Doob};
\[
     \E[U^X(a, b)]
=    \sum_{k=1}^\infty P[U^X(a, b) \geq k]
\leq \sum_{k=1}^\infty \frac{a^k}{b^k}
=    \frac{a}{b - a}
=    \E[U^X(a, b)],
\]
so the involved inequalities must in fact be equalities.
\end{proof}

\subsection{Davis' Lemma}
\label{app:davis-lemma}

We do not reproduce Davis' proof in detail.
It needs to be adapted to the martingale setting,
which is quite cumbersome to do.
Below we give an outline of the proof.

\begin{proof}[Proof sketch for \autoref{lem:alternations-upcrossings}]
\label{prf:alternations-upcrossings}
This proof relies on the observation that
the probability that $X_t$ rises (falls) by at least $\alpha$
does not decrease as $\alpha$ decreases.
Formally,
we argued in the proof of \autoref{thm:upper-bound}
that $P(X_T \geq x + \alpha \mid X_t = x) \leq \frac{x}{x + \alpha}$
using the \hyperref[thm:optional-stopping]{Optional Stopping Theorem}.
Since $(X_t)_{t \in \SetN}$ is bounded by $1$ from above,
the same argument can be carried out for the process $(1 - X_t)_{t \in \SetN}$,
giving an analogous bound
$P(X_T \leq x - \alpha \mid X_t = x) \leq \frac{1 - x}{1 - x + \alpha}$
when $(X_t)_{t \in \SetN}$ decreases.
These bounds are tight.

The idea of the proof is to define a martingale process $(Z_t)_{t \in \SetN}$;
the process defined by $Y_t := X_t + Z_t$ is then a martingale.
We need to show that $(Y_t)_{t \in \SetN}$ has the desired properties:
$0 \leq Y_t \leq 1$ and
the probability of having at least $2k$ $2\alpha$-alternations of $(X_t)_{t \in \SetN}$
does not exceed the probability of having at least $k$ $\alpha$-upcrossings of $(Y_t)_{t \in \SetN}$.

There are two sources of misalignment between
$2\eps$-alternations and $\eps$-upcrossings;
we consider them in turn.

First, drift:
if the martingale $(X_t)_{t \in \SetN}$ overshoots the target
and becomes larger than $X_{T'_{2k+1}} + \alpha$
or smaller than $X_{T'_{2k}} - \alpha$,
it changes the target in subsequent alternations.
Without loss of generality, consider the first case.
Suppose we are in time step $t$, have observed $u \in \Sigma^t$ and
$2k + 1$ alternations have been completed,
i.e., $T'_{2k+1} \leq t < T'_{2k+2}$.
By observing a symbol $a \in \Sigma$,
we would have $X_{t+1}(ua) \geq X_{T'_{2k+1}}(u) + \alpha$
with a possible overshoot $\gamma := X_{t+1}(ua) - (X_{T'_{2k+1}} + \alpha)$.
To compensate, we set $Z_{t+1}(ua) = Z_{t}(u) - \gamma$ and
$Z_{t+1}(ub) \geq Z_t$ appropriately for $b \in \Sigma \setminus \{ a \}$ such that
$Z_t$ fulfills the martingale condition (b) of \autoref{def:martingale}.
Removing the overshoots from the martingale
makes upcrossings and alternations coincide,
i.e. $A_t^Y(\alpha) = 2 U_t^Y(c - \alpha/2, c + \alpha/2)$
for a suitable constant $c$, which we will discuss below.
According to the aforementioned observation,
the new martingale is at least as likely to complete the alternation as the old one,
since we have reduced the distance needed to be traveled.

Second, the initial value $Y_0$.
Let $c \in (\alpha/2, 1 - \alpha/2)$ and
define $Z_0 := c + \alpha/2 - X_0$.
The constant $c$ denotes the center of the alternations,
i.e., $Y_n$ alternates between $c - \alpha/2$ and $c + \alpha/2$,
since $Y_0 = X_0 + Z_0 = c + \alpha/2$.
What value should we assign to $c$?
Since we only care about cases where the number of alternations is even,
$c = 1/2$ maximizes the probability of successful upcrossings~\cite[Lem.\ 7]{Davis:13}.
This intuitively makes sense:
there is an equal number of up- and downcrossings and
the probability of each of them being successful depends on the process'
distance from $0$ or $1$ respectively.

At this point we have a martingale process $(Y_t)_{t \in \SetN}$
that is bounded between $0$ and $1$, and
upcrossings and alternations coincide:
$A_t^Y(\alpha) = 2 U_t^Y(\frac{1 - \alpha}{2}, \frac{1 + \alpha}{2})$.
It remains to show that the probability of at least $k$ alternations
has not decreased compared to the process $(X_t)_{t \in \SetN}$.
By construction, this is already to case for single down- and upcrossings.
However, there could be cases where
the drift that we removed from the process would cause us
to move to a region where successful alternations are more likely.
But since we centered the process optimally,
this is not possible.

There is one other technical problem that we glossed over:
we have to make sure that
the process $(Y_t)_{t \in \SetN}$ exceeds neither $0$ nor $1$;
We have to stop the process at these points.
Moreover, if the process $Y_t$ reaches $1 + Z_t$ or
$Z_t$ but its value is in $(0, 1)$ instead of stopping it,
we switch to a random walk until we `get back on track'.
%
\renewcommand{\qedsymbol}{$\Diamond$}
\end{proof}
\fi

\end{document}